\makeatletter\@input{supp.aux.tex}\makeatother
\documentclass[10pt,twocolumn,letterpaper]{article}

\usepackage{sty/vl_header}
\usepackage{sty/cvpr}
\usepackage{times}
\usepackage{epsfig}
\usepackage{graphicx}
\usepackage{amsmath}
\usepackage{amssymb}
\usepackage{amsthm}

\newtheorem{theorem}{Theorem}

\newtheorem{assumption}{Assumption}

\usepackage[pagebackref=true,breaklinks=true,letterpaper=true,colorlinks,bookmarks=false]{hyperref}

\cvprfinalcopy

\newcommand{\ours}{RelEx}
\newcommand{\lonenorm}{$\ell_1$-norm}
\newcommand{\ltwonorm}{$\ell_2$-norm}
\newcommand{\linfnorm}{$\ell_\infty$-norm}
\newcommand{\supp}{the supp.}

\begin{document}

\title{
    Building Reliable Explanations of Unreliable Neural Networks: Locally Smoothing Perspective of Model Interpretation
}

\author{
    {Dohun Lim}
    \qquad{Hyeonseok Lee}
    \qquad{Sungchan Kim\thanks{Correspondence to: Sungchan Kim ({\tt{s.kim@jbnu.ac.kr}}).}}\\
    Division of Computer Science and Engineering, Jeonbuk National University, Korea\\
}

\maketitle
\thispagestyle{empty}

\begin{abstract}

    We present a novel method for reliably explaining the predictions of neural networks.
    We consider an explanation reliable if it identifies input features relevant to the model output by considering the input and the neighboring data points.
    Our method is built on top of the assumption of smooth landscape in a loss function of the model prediction: locally consistent loss and gradient profile.
    A theoretical analysis established in this study suggests that those locally smooth model explanations are learned using a batch of noisy copies of the input with the L1 regularization for a saliency map.
    Extensive experiments support the analysis results, revealing that the proposed saliency maps retrieve the original classes of adversarial examples crafted against both naturally and adversarially trained models, significantly outperforming previous methods.
    We further demonstrated that such good performance results from the learning capability of this method to identify input features that are truly relevant to the model output of the input and the neighboring data points, fulfilling the requirements of a reliable explanation.

\end{abstract}

\section{Introduction test}

\begin{figure*}[t]
    \begin{center}
        \includegraphics[width=1.0\linewidth]{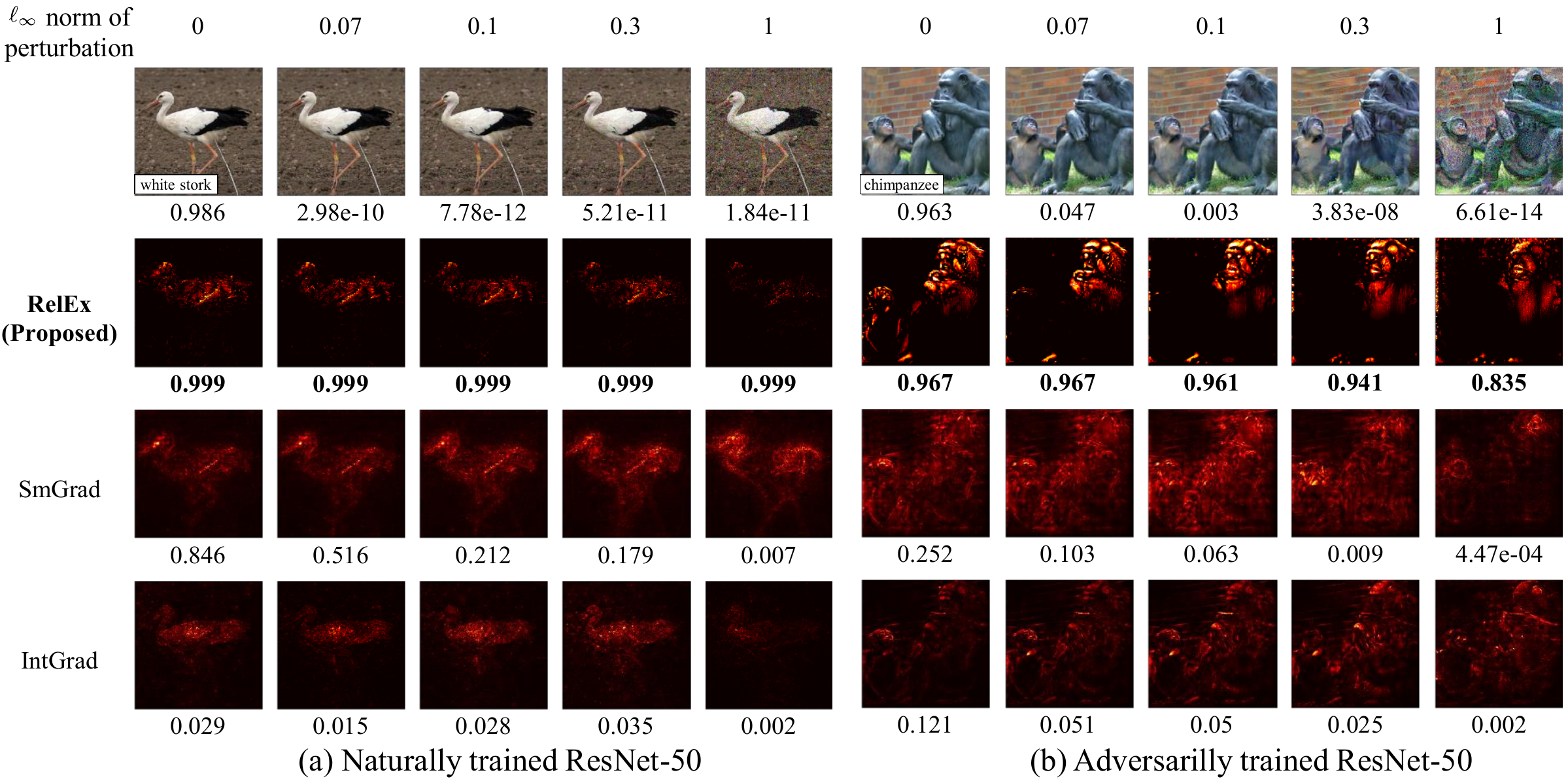}
    \end{center}
    \vspace*{-5mm}
    \caption{
        \textbf{Comparison of saliency maps.}
        Images on the top row depict adversarial examples created by a PGD attack~\cite{madry2018towards} for given perturbation distance in the \lonenorm~on (a) naturally trained and (b) adversarially trained ResNet-50.
        Numbers below images represent its softmax scores while ones below saliency maps indicate the scores of corresponding explanations.
        Our method generates saliency maps that leads to the scores close to 1 consistently in the presence of perturbations.
        The saliency maps of other methods, SmGrad~\cite{smilkov2017smoothgrad} and IntGrad~\cite{sundararajan2017axiomatic}, are visually plausible but irrelevant to the score.
        (See \supp~\ref{supp_qualitative_results} for more results)
    }
    \vspace*{-2mm}
    \label{fig_key_results}
\end{figure*}

The recent progress of deep neural networks has led to their adoption in various decision-critical applications, including medical, finance, and legal fields and autonomous vehicles.
However, the high modeling capacity of deep models renders the inner operations of the models mostly uninterpretable and demands human-understandable explanations for model predictions.
For this purpose, the model output attribution for input features is a popular idea.
The attribution aims to identify the importance of input features using end-to-end relationships between inputs and model predictions.
\black{We use a \emph{saliency map} as the visual form of an \emph{explanation}.}
A saliency map is a common approach to visual tasks to implement pixel-level or regional attributions for a given image~\cite{springenberg2014striving, selvaraju2017grad, fong2017interpretable, wagner2019interpretable, kapishnikov2019xrai, rebuffi2020there}.

The susceptibility of neural networks can cause false predictions for a given imperceptibly modified image~\cite{goodfellow2014explaining}, which has emerged as a new challenge in explaining model predictions~\cite{fong2017interpretable, dabkowski2017real, wagner2019interpretable, ghorbani2019interpretation, dombrowski2019explanations, heo2019fooling, slack2020fooling, subramanya2019fooling}.
For instance, recent work has demonstrated that \black{input can be manipulated, resulting in different saliency maps} without damaging the classification accuracy~\cite{ghorbani2019interpretation, dombrowski2019explanations, heo2019fooling, subramanya2019fooling}.
Such a false explanation is primarily due to the fragility of learned models that have highly nonsmooth decision boundaries rather than due to the explanation methods.
Although adversarial training has addressed these concerns~\cite{madry2018towards, zhang2019theoretically, schmidt2018adversarially, fawzi2018adversarial, gilmer2018adversarial, tsipras2018robustness, moosavi2019robustness}, it is not always applicable and still incomplete.

This discussion indicates the need to build a reliable model explanation with two requirements \black{if the goal is \emph{to recover input features that are important in a local neighborhood}:}
first, a reliable explanation method should be robust so that it generates consistent explanations along with neighboring (and thus similar) data points;
second, explanations generated by the method must have high fidelity for model predictions.

We propose \ours, a novel method to \emph{reliably} explain predictions of neural network-based classifiers.
\ours~aims to generate \emph{robust} and yet \emph{accurate} saliency maps of pixel-level importance for a given image.
Inspired by recent work on adversarial training~\cite{melis2018towards, madry2018towards, moosavi2019robustness}, we built \ours~on top of an assumption on a \emph{locally smooth explanation} for the vicinity of the input.

Although substantial work has proposed creating saliency maps based on gradient~\cite{simonyan2013deep, sundararajan2017axiomatic, smilkov2017smoothgrad, bach2015pixel, shrikumar2017learning, selvaraju2017grad} or perturbation~\cite{Petsiuk2018rise, fong2017interpretable, wagner2019interpretable, kapishnikov2019xrai, rebuffi2020there}, researchers have hardly addressed both robustness and accuracy in a single method.
Existing methods often \black{fail to find out essential features of the input and neighborhood for the model predictions} despite being visually plausible, as shown in this study.
Moreover, using either a random or adversarial perturbation of data points can manipulate their explanations~\cite{ghorbani2019interpretation, dombrowski2019explanations}.
Recent work has addressed such an issue partially~\cite{dabkowski2017real, fong2017interpretable, wagner2019interpretable, dombrowski2019explanations}.
In contrast to the existing methods, we construct an analysis to characterize the behavior of the proposed explanation method based on the assumption of the local explanation.
Specifically, the contributions of this paper are as follows.

\begin{itemize}
    \item
          We establish a quadratic approximation on a locally smooth landscape from the model explanation perspective to identify a trade-off between the accuracy and robustness of saliency maps.
          Our analysis reveals that the robustness of an explanation method is better achieved at the cost of the reduced accuracy of the model explanation and that the curvature of a loss function for learning a saliency map is inverse-proportional to the \lonenorm~of saliency maps to which, however, the explanation accuracy is proportional.
          A similar trace-off was investigated in adversarial training~\cite{tsipras2018robustness, zhang2019theoretically}; but it was hardly addressed in the context of an explanation method.
          Although the use of the \lonenorm~is not new this work is the first attempt to address its effects on building a reliable explanation method to the best of our knowledge.
    \item
          Our analysis leads to an easy-to-implement objective function to learn a saliency map by using backpropagation.
          We need only noisy copies of an input image as a batch for the optimization that is regularized by the \lonenorm~of a saliency map.
    \item
          \ours~\black{identifies input features relevant to a decision for all points in a neighborhood over the existing methods} when applying it to naturally and adversarially trained models.
          We demonstrate that explanations by the proposed method achieved a remarkably robust retrieval of the target classes from adversarial examples created via strong white-box attacks (Figure~\ref{fig_key_results}).
          Extensive evaluations indicate that such an advantage our method is due to learning appropriate saliency maps even with severe perturbations.
\end{itemize}

\section{Related Work}

We briefly review previous work related to the explanation of neural networks to highlight the benefit of this approach.
We first focus on methods to generate saliency maps and then on the literature on robust model explanations.

\textbf{Gradient-based methods.}
Existing explanation methods generate the importance of input features primarily based on the gradient or perturbation of an input image.
Gradient-based approaches measure individual feature importance as the sensitivity of input features concerning changes in the model prediction by using standard backpropagation~\cite{simonyan2013deep, smilkov2017smoothgrad, sundararajan2017axiomatic, bach2015pixel, selvaraju2017grad}.
The pixel-level gradient has inherent limitations, such as being noisy and saturation, to explain the model output~\cite{bach2015pixel, smilkov2017smoothgrad, sundararajan2017axiomatic, shrikumar2017learning}.
A method, called SmoothGrad, generates explanations by averaging gradient-based saliency maps of noisy copies of an input image~\cite{smilkov2017smoothgrad}, partially addressing adversarial attacks~\cite{heo2019fooling, dombrowski2019explanations}.
An approach in~\cite{sundararajan2017axiomatic} takes integrated gradients along with a straight path from the baseline value to each of the input features as an attribution of the particular feature.
Layer-wise relevance propagation and DeepLIFT back-propagate the model output by distributing it through a neural network according to neuronal activation~\cite{bach2015pixel, shrikumar2017learning}.
Although these principled approaches provide improved saliency maps, one can manipulate input images for the model to classify the images correctly but result in their saliency maps differently from original ones~\cite{ghorbani2019interpretation, dombrowski2019explanations, heo2019fooling, slack2020fooling, subramanya2019fooling}.

To summarize, because these methods merely react to the interactions of the model with the input and thus are unsupervised processes in nature, they are restricted to presenting the reaction of the model to the change.

\textbf{Perturbation- and activation-based methods.}
Another approach generates saliency maps as a change in the model output caused by perturbing the input image~\cite{dabkowski2017real, fong2017interpretable, Petsiuk2018rise, wagner2019interpretable, kapishnikov2019xrai, rebuffi2020there}.
These methods learn feature importance of an input at the pixel-level~\cite{wagner2019interpretable} like the proposed or at the regional basis ~\cite{dabkowski2017real, fong2017interpretable, Petsiuk2018rise, kapishnikov2019xrai, rebuffi2020there}.
Perturbation, such as occlusion and masks, is queried to an image repeatedly to learn an optimal saliency map.
Some incorporated regularizers for adversarial defense, unlike gradient-based methods~\cite{dabkowski2017real, fong2017interpretable, wagner2019interpretable}.
We demonstrate that these defenses are insufficient and can be deceived by carefully crafted adversaries.

Activation-based approaches combine activations of convolutional layers linearly to translate the spatial information of features maps in different layers to saliency maps~\cite{selvaraju2017grad, dabkowski2017real, rebuffi2020there}.
However, their saliency maps are diffused and are unreliable against the perturbation of input images.

\textbf{Robust explanation.}
Recent work has indicated that the susceptibility of neural networks is caused by a high modeling capacity with numerous parameters and a kinky landscape of the gradient of the model output due to the nonlinearity of the models~\cite{ghorbani2019interpretation, dombrowski2019explanations}.
The adversarial training of models incorporating robustness into the model during the training process is an active research area~\cite{ross2017improving, fawzi2018empirical, qin2019adversarial, moosavi2019robustness}.
Their main idea is to encourage robustness by enforcing a locally linear landscape of the model prediction for neighboring examples~\cite{shaham2015understanding, ross2017improving, moosavi2019robustness}.

We regard the notion of a locally smoothing behavior as a requirement of a reliable explanation.
Thus, explanations with saliency maps for the vicinity of an input should be similar and lead to the same class.
The authors in \cite{dombrowski2019explanations} demonstrated that piece-wise nonlinearity due to a rectified linear unit (ReLU) of neural networks might mislead to a wrong explanation, even with a correct prediction.
They proposed to use SoftPlus instead of the ReLU in the model, which is identical to using SmoothGrad~\cite{smilkov2017smoothgrad}, providing partial tolerance to an adversarial attack~\cite{dombrowski2019explanations, heo2019fooling}.
An analysis established in~\cite{wagner2019interpretable} indicated that the robustness of the model is degraded by numerous parameters and the largest singular value of the Hessian regarding the parameters that represents the curvature of the gradient landscape.
Despite previous efforts, no tangible realization of a robust explanation method has been addressed.
Some effort has been made to improve the quality of saliency maps, considering adversarial defense~\cite{dabkowski2017real, fong2017interpretable, wagner2019interpretable, dombrowski2019explanations} by
reducing a total variation~\cite{dabkowski2017real, fong2017interpretable} or modifying the activation function of neurons~\cite{wagner2019interpretable, dombrowski2019explanations}.
They, however, result in insufficient defense capability.
\ours~exploits the analysis established in this work to address the limitations above efficiently and effectively, and non-intrusively through a simple optimization framework.

\section{Proposed Approach}

This section describes the details of \ours.
We begin by stating two assumptions that a reliable explanation method should satisfy.
Then, we analyze the bounds on the robustness of the proposed method in response to these requirements in Section~\ref{subsec_local_smoothness_analysis}.
A cost function of \ours~for learning a saliency map is formulated in Section~\ref{subsec_optimization}.

\textbf{Notations.}
A neural network-based classification model maps an input image $x_0 \in \mathbbm{R}^d$ to an output $y \in [0,1]^{|C|}$, where $y$ is a vector representing the softmax scores of a specified set of classes $C$.
We define $f_c(x_0)$ as the probability of $x_0$ being classified as $c \in C$.
A \emph{saliency map} $m_c \in [0,1]^d$ represents the importance of individual features (i.e., pixels of $x_0$) corresponding to the model prediction $f_c (x_0)$.
For simplicity, we use $m$ and $f(x_0)$ instead of $m_{c_T}$ and $f_{c_T}(x_0)$, respectively, when referring to the target class $c_T$ of $x_0$.

\subsection{Local Smoothness for Robust Explanation} \label{subsec_local_smoothness_analysis}

\black{Suppose a local interpretation of the model prediction refers to identify features relevant to $x_0$ and its local neighborhood.
    In that case, the interpretation} implies that explanations of data samples neighboring to $x_0$ should vary slowly, rendering a corresponding smooth landscape although it holds locally.
\black{According to the notion of the local interpretation}, we expect data points in the vicinity of $x_0$ to be in the same class with similar saliency maps.
We claim that, in particular, two conditions should be met in the local data points for the model explanation to be reliable by ensuring the local smoothness.

\begin{assumption} [\textbf{Locally consistent model prediction}] \label{assumption_label_robustness}
    For a given saliency map $m$ of input $x_0$ with a target class $c_T$, we consider data samples $\mathcal{D} = \{ x_i \}$ where $ \Vert x_i - x_0 \Vert_{p} \leqq \epsilon$ and $\epsilon$ is a small positive number, which we regard as the perturbation of $x_0$.
    Thus, $m \odot x_i$ for data points $x_i \in \mathcal{D}$ concerning $m$ would be correctly classified as $c_T$ (i.e., $ f(m \odot x_i) \approx f(m \odot x_0) $)\black{, where $\odot$ is an element-wise product of vectors.}
    We use the \ltwonorm~ or \linfnorm~as the perturbation distance.\footnote{We represent the \ltwonorm~of vector $a$ as $ \Vert a \Vert $.}
\end{assumption}

\begin{assumption} [\textbf{Locally consistent saliency maps}] \label{assumption_saliency_robustness}
    A saliency map $m^{(x_i)}$ for data point $x_i \in \mathcal{D}$ is similar to that of $x_0$, $m$, which leads to $ \Vert m - m^{(x_i)} \Vert \leqq \delta$ for a small positive number $\delta$.
\end{assumption}

\textbf{Local smoothness for label consistency.}
We consider a simple analysis to elucidate how a saliency map is related to the label consistency of an explanation method.
As the distance between the data points depends on the task, we use cross-entropy in this study.
Thus, given an input $x$ and a saliency map $m$, we denote a loss function of classifying $m \odot x$ as its target class $c_T$ as $L(x, m) = -\log f(m \odot x)$.
Inspired by the setting of the analysis established in~\cite{moosavi2019robustness}, we assume that $L(\cdot)$ is well approximated as a quadratic form at a sufficiently small distance $\gamma$, which is given by the following:
\begin{equation}\label{eq_loss_robust}
    L( x_0 + \gamma , m ) \approx L( x_0 , m ) + \nabla L ( x_0 , m )^T \gamma + \frac{1}{2} \gamma^T H \gamma
\end{equation}
where $\nabla L( x_0 , m )$ and $H$ denote the gradient and Hessian of $L(\cdot)$ at $ x_0 $, respectively.
The robustness of the explanation method is represented by $\gamma$ if we ensure that all samples in the $L_2$ ball of radius $\gamma$ centered at $x_0$ are classified correctly.
The second-order derivative term in Eq.~\eqref{eq_loss_robust} enables the investigation of the divergent curvature of the loss function from the perspective of the optimization landscape.

If all data points in the $L_2$ ball are classified correctly, it holds that $ L( x_0 + \gamma , m ) \leqq \tau $ and $ L( x_0 , m ) \leqq \tau $ where $\tau$ is the threshold for the input to be classified correctly.
For example, in the case of binary classification, $ \tau = -\log {\frac{1}{2}}$.
We measure the robustness $\gamma$ by evaluating its maximum $\gamma^+$ as follows:
\begin{equation}\label{eq_gamma_argmax}
    \gamma^+ = \argmax_{\gamma} \Vert \gamma \Vert~\text{s.t.}~L( x_0 + \gamma , m ) \leqq \tau .
\end{equation}
For the ease of calculation, we convert Eq.~\eqref{eq_gamma_argmax} to an equivalent minimization problem as follows:
\begin{equation}\label{eq_gamma_argmin}
    \gamma^+ = \argmin_{\gamma} \Vert \gamma \Vert~\text{s.t.}~L( x_0 + \gamma , m ) \geqq \tau .
\end{equation}
Let $ c = \tau - L ( x_0, m ) \geqq 0 $. Substituting $c$ into Eq.~\eqref{eq_gamma_argmin} yields
\begin{equation}\label{eq_gamma_argmin_2}
    \gamma^+ = \argmin_{\gamma} \Vert \gamma \Vert~\text{s.t.}~\nabla L ( x_0 , m )^T \gamma + \frac{1}{2} \gamma^T H \gamma \geqq c.
\end{equation}
The introduction of Eq.~\eqref{eq_gamma_argmin_2} is not to calculate $\gamma^+$ but to derive the lower bound of the robustness and the upper bound of the loss function as a function of $m$.
This analysis also holds when we use the \linfnorm~for the perturbation distance metric.

\begin{theorem}[\textbf{Local explanations with respect to label consistency}] \label{theorem_cls_robust}
    Let $\gamma = \alpha \cdot v$ where $ \Vert \gamma \Vert  = \alpha$ and $ \Vert v \Vert  = 1$.
    Let $\mathcal{D} = \{x_i\}$ be a set of data samples where $ \Vert x_i - x_0 \Vert  \leqq \epsilon$.
    For a given saliency map $m$ calculated from $x_0$, it holds that
    \black{
        \begin{equation}\label{eq_alpha_bound}
            \alpha \geqq \frac{c}{ \Vert m \Vert  _1} \cdot \frac{ 2 } {  \Vert  -g(x_0 + \alpha v) + g(x_0)  \Vert   + 2  \Vert  g(x_0)  \Vert  }
        \end{equation}
    }
    where
    \begin{align} \label{eq_g_x}
        g(x) = -\nabla L (x, m) = \nabla \log f(m\odot x).
    \end{align}
    It also holds that the loss function in Eq.~\eqref{eq_loss_robust} with respect to $x_0 + \gamma$ is upper-bounded as follows:
    \begin{equation} \label{eq_classification_loss_bound}
        \begin{split}
            &L(x_0 + \gamma, m ) \leqq \\
            &\alpha  \Vert m \Vert _1 \left( \frac{ \Vert -g(x_0 + \alpha v) + g(x_0) \Vert }{2}  + \Vert  g(x_0)  \Vert  \right).
        \end{split}
    \end{equation}
\end{theorem}

Theorem~\ref{theorem_cls_robust} has two implications.
First, the label consistency of the proposed explanation method is inversely proportional to the complexity of the saliency map that is represented  as $ \Vert m \Vert _1$ in Eq.~\eqref{eq_alpha_bound}.
Intuitively, images far from $x_0$ can be viewed as significantly perturbed copies of $x_0$.
Thus, a robust explanation applicable along with these data points is likely to contain a small number of features that are invariant to the perturbation.
Previous work on adversarial training has addressed this issue to explain the trade-off between robustness and accuracy, arguing that the high prediction performance of models is due to subtle features that are susceptible to perturbation~\cite{tsipras2018robustness}.
Our analysis reveals similar concerns in the context of the model explanation with the explanation complexity notion.
Figure~\ref{fig_key_results} illustrates two examples where saliency maps become more sparse (i.e., decreasing $ \Vert m \Vert _1$) to maintain the classification performance of the explanations as the images undergo further perturbations.
We further validate this observation using extensive experiments in Section~\ref{subsec_eval_label_robustness}.

Second, in contrast to the behavior of the robustness in the explanation, the classification accuracy for the neighboring data points is proportional to $||m||_1$ as shown in Eq.~\eqref{eq_classification_loss_bound}.
This equation is another representation of Eq.~\eqref{eq_alpha_bound} that agrees with the aforementioned trade-off of adversarial training.

\begin{figure}[t]
    \begin{center}
        \includegraphics[width=0.7\linewidth]{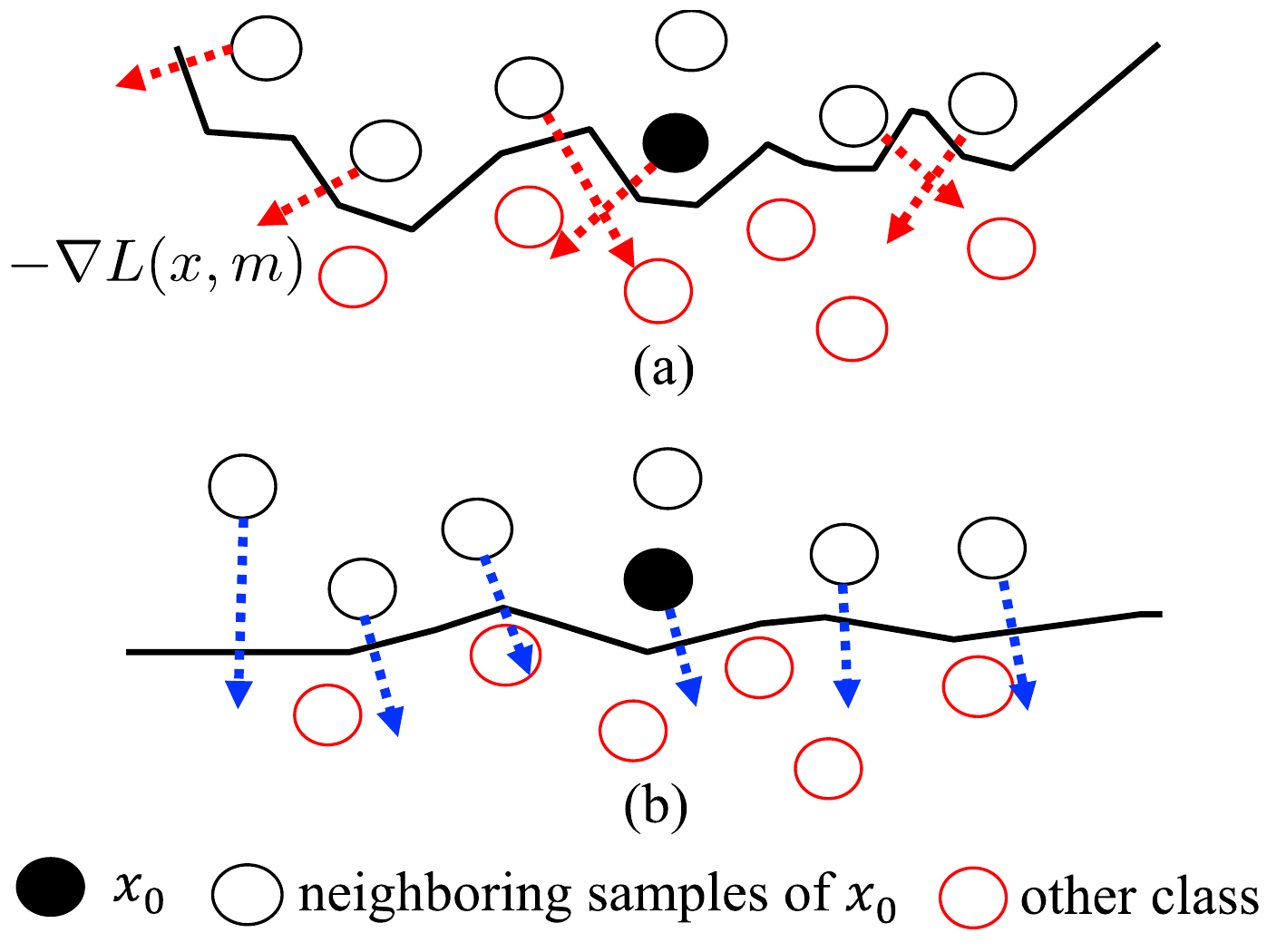}
    \end{center}
    \vspace*{-5mm}
    \caption{
        \textbf{Kinky gradient landscape.}
        Two landscapes on the classification loss lead to the identical classification accuracy.
        Given that a saliency map is dependent on the negative gradient of the loss function, the saliency map varies according to the landscape geometry.
        (a) Landscape is kinky.
        This results in divergent gradient profile of the neighboring data points and, thus, their inconsistent saliency maps, violating Assumption~\ref{assumption_saliency_robustness}.
        (b) Ensuring a smooth landscape provides similar saliency maps for the data points.
    }
    \vspace*{-2mm}
    \label{fig_kinky_gradient}
\end{figure}

\textbf{Local smoothness for consistency of saliency maps.}
We demonstrated the effects of the explanation complexity given by $ \Vert m \Vert _1$ on the robustness in terms of the classification accuracy.
However, it is still unclear how the explanation complexity affects the saliency map consistency among the adjacent data points.
Figure~\ref{fig_kinky_gradient} motivates this investigation, where two landscapes in the classification loss have identical accuracy but the explanations of the data points at the landscapes are quite different.

Given the assumption of the local consistency in model prediction, it holds that $f(m^{(x_i)} \odot x_i) \approx f(m \odot x_0)$, and, thus, $L(x_i, m^{(x_i)}) \approx L(x_0, m)$, where $m^{(x_i)}$ and $m$ are saliency maps for $x_i$ and $x_0$, respectively.
Then, it is possible to approximate $L(x_i, m^{(x_i)})$ using $L(x_0, m)$.
The following result reveals the relationship between $ \Vert m \Vert _1$ and the saliency map consistency along with the data points around the input data.

\begin{theorem}[\textbf{Local explanations with respect to saliency map consistency}] \label{theorem_sal_robust}
    Let $\mathcal{D} = \{ x_i \}$ be the vicinity of the input data $x_0$ such that $ \Vert x_i - x_0 \Vert  \leqq \epsilon$ where $\epsilon$ being a small positive number.
    Then, the distance between the gradients of explanations of $x_i$ and $x_0$ is lower-bounded as follows:
    \begin{equation} \label{eq_saliency_similarity_upper_bound}
        \Vert  \nabla L(x_i, m ) - \nabla L (x_0, m )  \Vert  \leqq  \Vert  m  \Vert _1 \cdot  \Vert  -g(x_0 + \alpha v) + g(x_0)  \Vert .
    \end{equation}
\end{theorem}

The distance between corresponding gradients represents similarity of saliency maps by referring to Eq.~\eqref{eq_g_x}.
Therefore, Theorem~\ref{theorem_sal_robust} indicates that the proposed method prefers a smaller value of $ \Vert m  \Vert _1$ for the saliency maps to be consistent, as is the classification robustness case.

The results of Theorem~\ref{theorem_cls_robust} suggest formulating \ours~as an optimization problem to consider the trade-off represented by Eq.~\eqref{eq_alpha_bound} and~\eqref{eq_classification_loss_bound}.
We denote an objective function $\mathcal{J} (\cdot)$ for the data points in the $\ell_p$ ball centered at the input image $x_0$ as follows:
\begin{equation}\label{eq_objective_function}
    \mathcal{J} ( \mathcal{D}, m ) = - \frac{1}{|\mathcal{D}|} \sum_i \log  f ( m \odot x_i).
\end{equation}
\ours~learns a saliency map $m^+$ given $\mathcal{D} = \{x_i\}$,
\begin{equation}\label{eq_base_optimization}
    m^+ = \argmin_{m} \mathcal{J}( \mathcal{D}, m ) + \lambda_1  \Vert m \Vert _1
\end{equation}
where $\lambda_1$ a regularization strength of $ \Vert m \Vert _1$ following the results of Eq.~\eqref{eq_alpha_bound} and~\eqref{eq_saliency_similarity_upper_bound}.
The use of $ \Vert m \Vert _1$ coincides with previous work that generated perceptually improved images~\cite{isola2017image, zhu2017unpaired, wang2018high}, which also applies to the proposed method.
The previous explanation methods have also learned a saliency map using $ \Vert m \Vert _1$ like our method~\cite{fong2017interpretable, wagner2019interpretable, dabkowski2017real}.
However, we demonstrated that using a batch $\mathcal{D}$ in conjunction with the regularizer increases the robustness considerably.
See \supp~\ref{supp_proofs} for proofs of the theorems.

\subsection{Optimization of the Proposed Method} \label{subsec_optimization}

\textbf{Fidelity of saliency maps and faithful explanations.}
Another critical aspect of a reliable explanation method is that a saliency map should represent \emph{essential regions} of the input~\cite{fong2017interpretable, melis2018towards, dabkowski2017real, wagner2019interpretable}.
While it is difficult to quantify the fidelity of a saliency map in general, we consider two definitions: \emph{the smallest susceptive region} and \emph{smallest evidential region}~\cite{fong2017interpretable, dabkowski2017real, chang2018explaining, du2018towards}.
The smallest susceptive region is the minimum area of an image that changes the model prediction when the region is altered.
The smallest evidential region refers to the minimum area to be preserved to maintain the model prediction.
Although these concepts appear similar, they are different, for example, in an image classified as ``dog'' containing two dogs.
From the smallest evidential region viewpoint, as long as the model explanation covers any of the dogs, it may lead to correct classification.
However, the explanation is not the smallest susceptive region because the part untapped by the explanation still has information concerning the target class.
Determining the smallest susceptive region can be viewed as identifying a background.

The objective function in Eq.~\eqref{eq_base_optimization} is likely to determine the smallest evidential region because it is advantageous in terms of reducing $||m||_1$.
Inspired by the above discussion, we incorporated an additional regularization term into the objective function to improve our explanation in the smallest susceptive region.
Given the batch of $\mathcal{D}$, we considered an additional loss, $\mathcal{B} ( \mathcal{D}, m )$, a classification loss for the counterpart region in $x_0$ with respect to $m$.
Thus, the objective function is given by the following:
\begin{equation} \label{eq_full_optimization}
    \begin{split}
        \mathcal{B} ( \mathcal{D}, m ) = - \frac{1}{|D|} \sum_{ x_i \in \mathcal{D} } \log \left( 1- f ((1-m) \odot x_i) \right), \\
        m^+ = \argmin_{m} \mathcal{J} ( \mathcal{D}, m ) + \lambda_1 ||m||_1 + \lambda_2 \mathcal{B} ( \mathcal{D}, m )
    \end{split}
\end{equation}
where $\lambda_2$ is the coefficient of $\mathcal{B} ( \mathcal{D}, m )$.

\begin{figure*}[t]
    \begin{center}
        \includegraphics[width=0.95\linewidth]{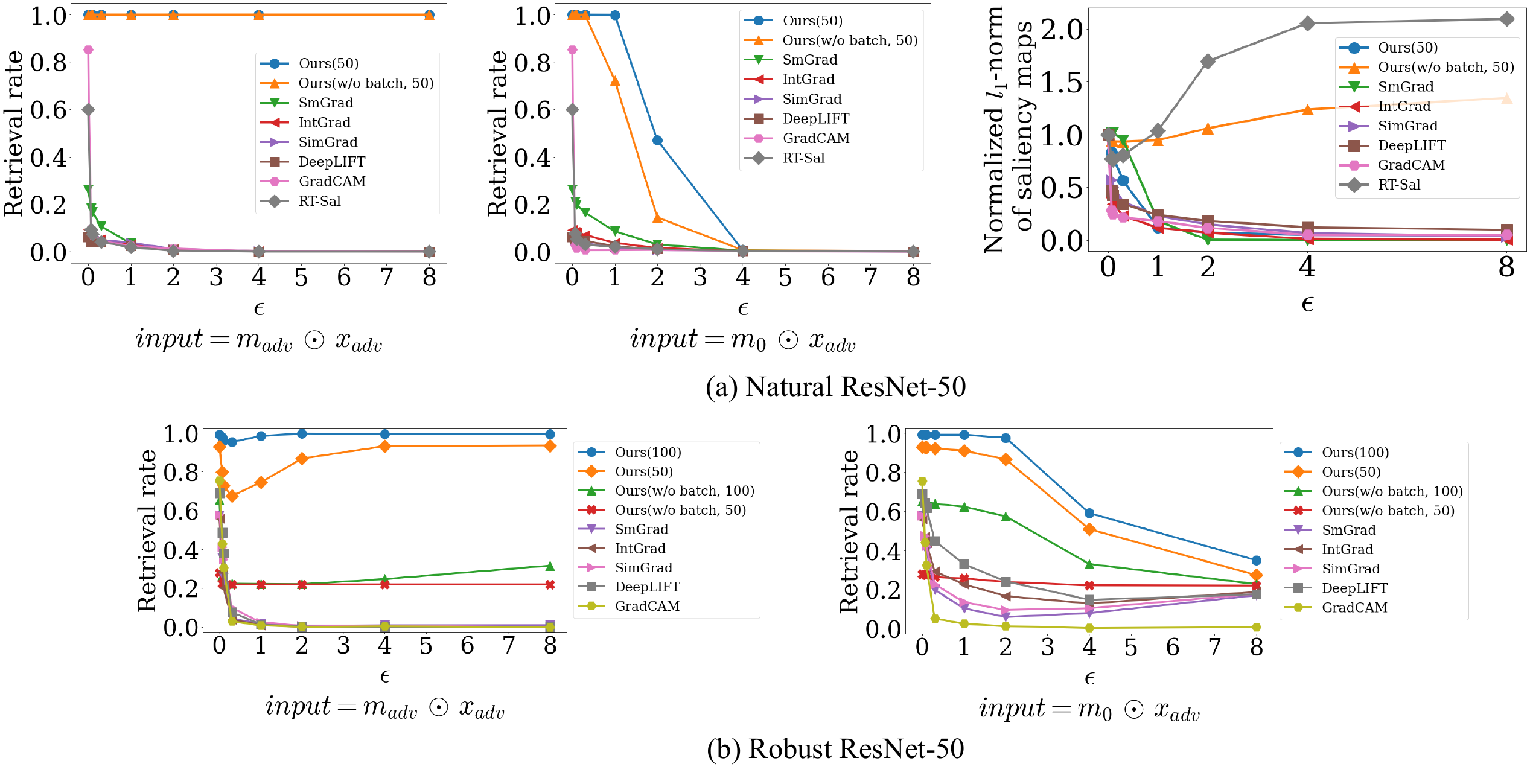}
    \end{center}
    \vspace*{-5mm}
    \caption{
        \textbf{Target class retrieval against the untargeted attacks on (a) the natural and (b) robust ResNet-50.}
        The inputs are presented below each of plots.
        Our method has three variants without using the batch (\emph{w/o batch}) or different epochs to iterate (\emph{50} and \emph{100}) when learning saliency maps.
        (Top-right) Plots present the average \lonenorm s of saliency maps of the adversarial examples by the methods on the natural model.
        The norms of the saliency maps are normalized to those of the clean images for each method.
    } \label{fig_class_robustness_untargeted_attack_resnset50}
    \vspace*{-2mm}
\end{figure*}
\section{Experimental Results}

We validate the proposed method through extensive experiments to answer the following questions:
\begin{itemize} [noitemsep,topsep=0pt]
    \item \emph {Are explanations using \ours~robust to retrieve the original target classes of input images against various adversarial attacks?}
    \item \emph{If it is the case, how do saliency maps learned by \ours~represent relevant evidence for model predictions compared with previous work?}
\end{itemize}
\black{The supp.}~\ref{supp_implementation_details} presents the implementation details.
\black{The code of \ours~is available at \url{https://github.com/JBNU-VL/RelEx}.}

\subsection{Robustness Evaluations with Class Retrieval}\label{subsec_eval_label_robustness}

We first demonstrate the robustness of this approach by evaluating the retrieval of the original target classes for given images against adversarial attacks~\cite{szegedy2013intriguing, goodfellow2014explaining} as performed in previous work~\cite{fong2017interpretable}.
Because the adversarial attacks fool models with only small perturbations to images, we use the attacks to validate the advantage of the proposed method.
We applied \emph{untargeted} and \emph{targeted attacks} based on a white-box threat model to the classifiers to create adversarial copies of the images.

We denote sampled clean images by $X_0$ and their adversarial counterparts by $X_{adv}$ for a given adversarial attack, respectively.
Also, $m_{adv}$ denotes a saliency map of $x_{adv} \in X_{adv}$.
Then, we measure the rate of the successful retrieval of the target class of $x_0$ as the most likely class of $x_{adv}$ for a given input, $m_{adv} \odot x_{adv}$, which is
$\mathop{\mathbbm{E}} \left [ \mathbbm{I} \{ \argmax_{c} f_c( m_{adv} \odot x_{adv} ) = c_{x_0} \} \right ]$
where $c_{x_0}$ is the target class of $x_0$, and $\mathbbm{I}(\cdot)$ is an indication function.

We evaluate two types of inputs: one is what is described above (i.e., $m_{adv} \odot x_{adv}$) and the other using a saliency map of a clean image (i.e., $m_{0} \odot x_{adv}$).
The evaluation of the former reveals whether a method extracts a saliency map accurately in the presence of perturbation.
The latter evaluates the robustness of saliency maps of clean images.

\textbf{Adversarial attack methods.}
The untargeted attack manipulates an input image \emph{to lead model predictions to arbitrary false labels}.
We applied \ours~to pretrained ResNet-50~\cite{he2016deep} using ImageNet~\cite{deng2009imagenet}.
We also considered a robust counterpart of the model that was adversarially trained~\cite{tsipras2018robustness}.
We denote them as \emph{natural} and \emph {robust} models, respectively.
We expect that adversarial examples crafted against the robust model are more difficult to defend than those from the natural model and, therefore, can better evaluate explanation methods.

The targeted attack aims to change an image to mislead to a specified false class or saliency map.
We chose to \emph{create false saliency maps of given images against a given explanation method with their classes kept} comparing with previous work~\cite{dombrowski2019explanations, ghorbani2019interpretation}.
We used two strategies for manipulating saliency maps: \emph{unstructured} and \emph{structured} attacks.
The former aims to create uninformative saliency maps whereas the latter misleads to the saliency map of a specified class.

\textbf{Setups for evaluation using the untargeted attack.}
For the natural ResNet-50, we randomly sampled 4000 images from the validation set of ImageNet.
The robust ResNet was trained on the Restricted ImageNet, a customized subset of ImageNet~\cite{tsipras2018robustness}.
We additionally sampled 1000 random images from the validation set of the dataset for the robust model.
Then, we created corresponding adversarial images using the projected gradient descent (PGD), one of the best universal first-order adversarial attacks~\cite{madry2018towards} with the configuration for the MNIST dataset.
We varied the \linfnorm~of the perturbation distance to $\{0.07, 0.1, 0.3, 1, 2, 4, 8\}$.
We compared our approach with the followings: \emph{SimGrad}~\cite{simonyan2013deep}, \emph{SmGrad}~\cite{smilkov2017smoothgrad}, \emph{IntGrad}~\cite{sundararajan2017axiomatic}, \emph{DeepLIFT}~\cite{shrikumar2017learning}, \emph{RT-Sal}~\cite{dabkowski2017real}, and \emph{GradCam}~\cite{selvaraju2017grad}.
See \supp~\ref{supp_adversary_generation} for the details on the adversarial example generations.

\textbf{Results of the untargeted attack.}
We present the results in Figure~\ref{fig_class_robustness_untargeted_attack_resnset50}.
First, our explanations successfully retrieve the original target classes along with the entire perturbation distance.
This result shows that our approach extracts meaningful evidence robustly in the presence of severe perturbations.
Other methods encountered significant performance decrease even at the smallest perturbation, $\epsilon=0.07$, which is visually imperceptible.
This is probably because the subtle features in the clean images were mostly perturbed even with such a small value of $\epsilon$; thus, the methods failed to determine robust features invariant to the perturbations.
Second, our explanation learned from the clean images contains robust features such that it applies to the adversaries up to $\epsilon=1.0$.
This allows a model to be tolerant against such an attack without adversarial training when combined with \ours.
In contrast, other approaches performed poorly, similar to the previous case.

For the ablation study, we considered four variants of our method without the batch or varying the number of epochs to iterate for solving Eq.~\eqref{eq_full_optimization}.
The benefit of the batch and iterating more epochs increased the robustness of explanations on clean images against the attack, which is more significant on the robust model.

Figure~\ref{fig_class_robustness_untargeted_attack_resnset50} also depicts that the higher perturbation results in the smaller \lonenorm s of our saliency maps.
Although we observe a similar behavior with the gradient-based approaches, they failed to determine relevant features, performing poorly in the class retrieval.
These results demonstrate the benefit of smoothing the local explanation as discussed in Section~\ref{subsec_local_smoothness_analysis}.

\begin{figure}[t]
    \begin{center}
        \includegraphics[width=\linewidth]{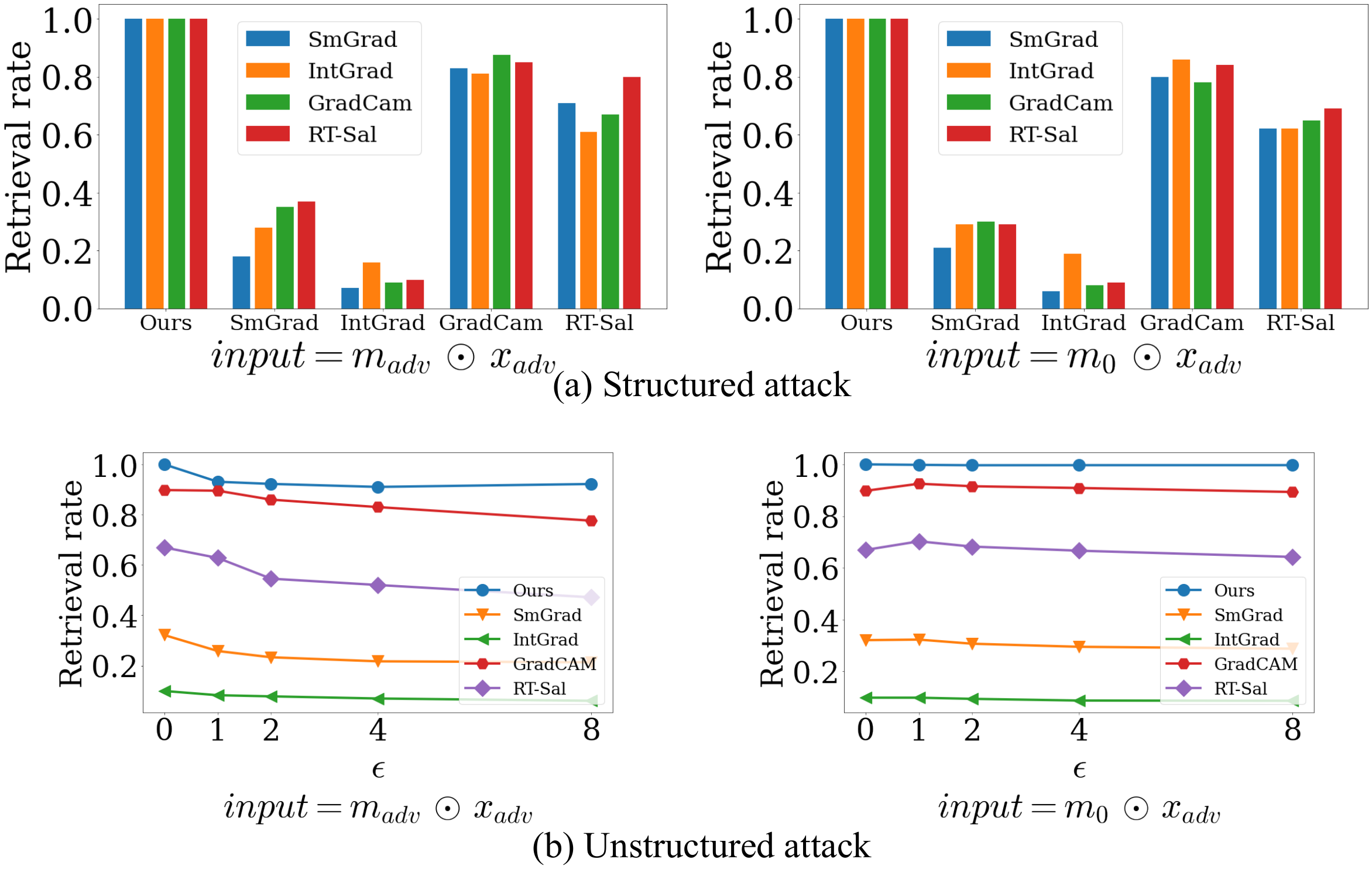}
    \end{center}
    \vspace*{-5mm}
    \caption{
        \textbf{Target class retrieval performance against the targeted attacks on the natural ResNet-50.}
        The explanation types are presented below each plot.
        (a) The horizontal axis represents the explanation methods to compare with each other.
        Each colored bar in the plots indicates a method to which the attack is applied.
        (b) Plots correspond to the adversarial images against RT-Sal~\cite{dabkowski2017real}.
    }
    \vspace*{-2mm}
    \label{fig_class_robustness_targeted_attack_resnset50}
\end{figure}

\textbf{Setups for evaluation using the targeted attacks.}
We created 1000 adversarial examples of the sampled input images by applying the targeted attacks to each method on the natural ResNet-50 as proposed in~\cite{dombrowski2019explanations} and~\cite{ghorbani2019interpretation} for the structured and unstructured attacks, respectively.
However, we observed that no adversarial examples were generated against our method.
Although a correct analysis of this observation is beyond the scope of this study, we propose that it is because of the insufficient perturbations of the targeted attacks to mislead our method.
We empirically validated the assumption that the \ltwonorm~of saliency maps due to the targeted attacks belongs to the region where \ours~is reliable in the untargeted attack. (See \supp~\ref{supp_targeted_perturbation_distance})
Instead, we evaluated our method using the adversarial examples against other methods by assuming that the adversarial attack is transferable~\cite{heo2019fooling}.

\textbf{Results of the targeted attacks.}
Overall, the results of the targeted attacks are similar to those of the untargeted attack, as depicted in Figure~\ref{fig_class_robustness_targeted_attack_resnset50}.
As expected, we observed that the attacks are transferable.
\ours~achieved an outstanding retrieval rate of close to 1 over all the settings.
Unlike the results of the untargeted attacks, the retrieval rates with GradCAM and RT-Sal, which are about 0.83 and 0.60, respectively, are comparable to the rate of the proposed method.
This suggests that the PGD-based untargeted attack is more effective than the targeted attacks.
Although Figure~\ref{fig_class_robustness_targeted_attack_resnset50} provides the results against RT-Sal in the unstructured attack, we found similar observations in the attacks against other methods.
See the supp.~\ref{supp_targeted_attack_result} for more results.

\subsection{Evaluations of the Fidelity of Saliency Maps}

\begin{figure}[t]
    \begin{center}
        \includegraphics[width=1.0\linewidth]{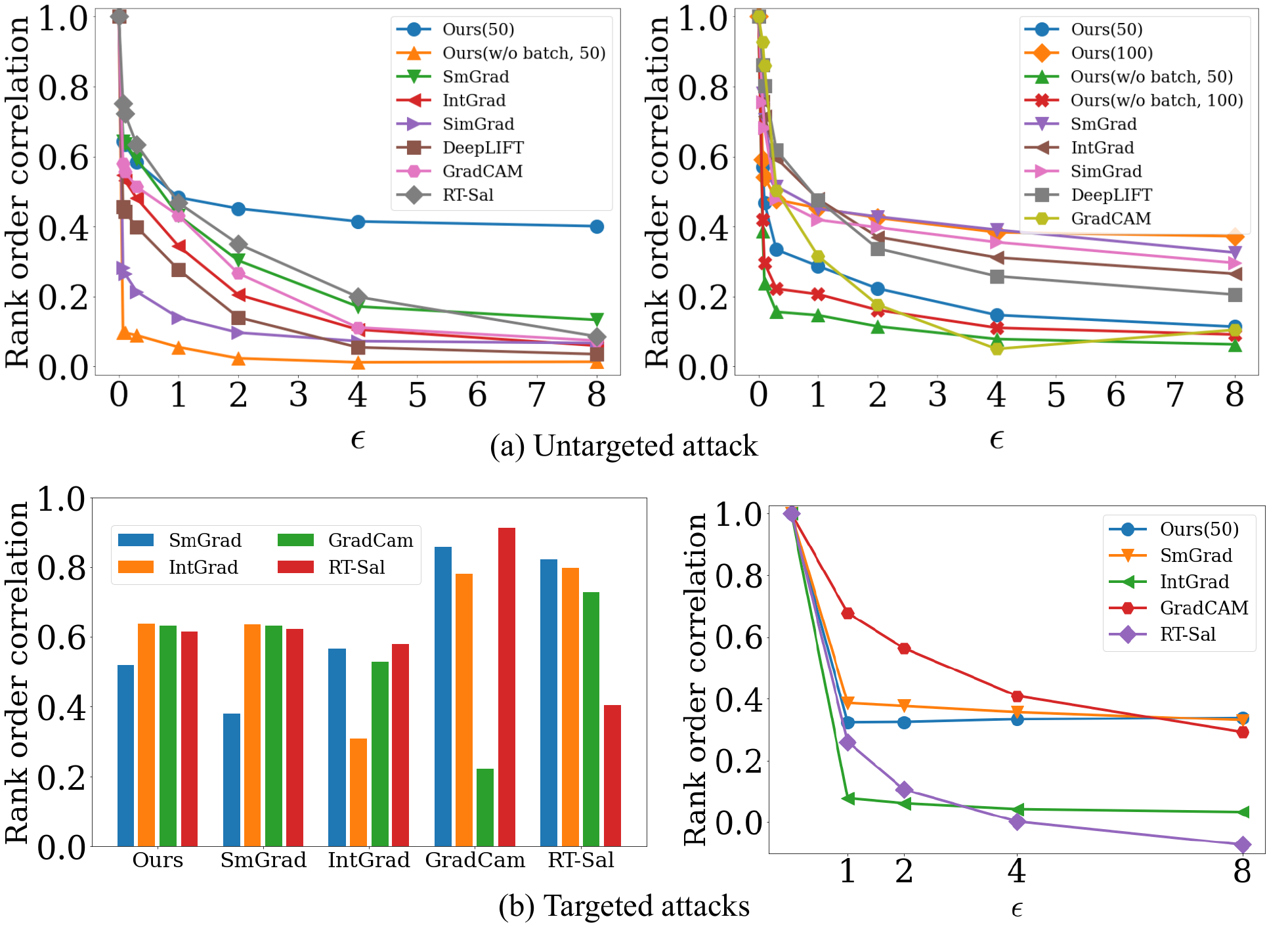}
    \end{center}
    \vspace*{-5mm}
    \caption{
        \textbf{Similarity of the saliency maps in the rank order correlation for (a) untargeted and (b) structured attacks.}
        (Bottom-left) See Figure~\ref{fig_class_robustness_targeted_attack_resnset50}(a) for the axis labels.
        (Bottom-right) The adversarial images were created against RT-Sal~\cite{dabkowski2017real}.
    }
    \vspace*{-2mm}

    \label{fig_sliency_spatial_similarity}
\end{figure}

\textbf{Metrics for the similarity of saliency maps.}
The evaluations of the target class retrieval demonstrated the robustness of the learned explanations by \ours.
To understand why, we delved into the fidelity of the saliency maps.
In particular, we quantified the quality of the saliency maps of adversarial examples by measuring 1) their spatial similarity to those of their clean counterparts and 2) the relevance of features identified by an explanation to a class score.

In the similarity evaluation, we use a metric, \emph{Spearman's rank-order correlation}~\cite{spearman1987proof}, to evaluate the similarity of saliency maps as in~\cite{ghorbani2019interpretation, dombrowski2019explanations}.
The Spearman's rank-order correlation inherently ranks the importance of input features according to a saliency map and enables us to naturally correlate feature ranks between saliency maps.

\textbf{Results of the similarity of saliency maps.}
Figure~\ref{fig_sliency_spatial_similarity} indicates that the spatial similarity of \ours~is analogous to those of other methods, unlike the case of the class retrieval.
Worse, all other methods in the case of the untargeted attack outperform our method without the batch.
The results suggest that the similarity metric is inadequate to explain the outstanding performance of \ours~in the class retrieval.
The performance mismatch between the class retrieval and the similarity of saliency maps by existing methods is due to the incorrect attribution of input features.
We observed that \ours~learned a saliency map adapting to the degree of perturbations.
In contrast, the existing methods failed to attribute \black{relevant input features to the neighboring data points}, and their saliency maps remain fixed, as illustrated in Figure~\ref{fig_key_results}(b).
See \supp~\ref{supp_saliency_map_similarity} for results of another metric.

\begin{table}[t]
    \centering
    \caption{
        \textbf{Comparison of the feature relevance, $R$, of saliency maps for given \linfnorm~of perturbation, $\epsilon$, against the untargeted attack.}
    }
    \resizebox{\linewidth}{!}{
        \begin{tabular}{|c|c|c|c|c||l|l|l|l|}
            \hline
            \multicolumn{1}{|l|}{} & \multicolumn{4}{c||}{Natural ResNet-50} & \multicolumn{4}{c|}{Robust ResNet-50}                                                                                                 \\
            \hline
            $\epsilon$             & 0.07                                    & 0.1                                   & 0.3           & 1.0           & 0.07          & 0.1           & 0.3           & 1.0           \\
            \hline \hline
            SimGrad                & 0.03                                    & 0.03                                  & 0.04          & 0.03          & 0.40          & 0.37          & 0.22          & 0.14          \\
            \hline
            IntGrad                & 0.07                                    & 0.07                                  & 0.06          & 0.03          & 0.35          & 0.30          & 0.13          & 0.07          \\
            \hline
            SmoothGrad             & 0.18                                    & 0.17                                  & 0.12          & 0.04          & 0.38          & 0.33          & 0.15          & 0.10          \\
            \hline
            DeepLIFT               & 0.04                                    & 0.04                                  & 0.04          & 0.04          & 0.48          & 0.43          & 0.20          & 0.12          \\
            \hline
            GradCAM                & 0.14                                    & 0.11                                  & 0.08          & 0.03          & 0.58          & 0.45          & 0.39          & 0.16          \\
            \hline
            RT-Sal                 & 0.12                                    & 0.10                                  & 0.06          & 0.02          & N/A           & N/A           & N/A           & N/A           \\
            \hline
            MASK                   & 0.05                                    & 0.04                                  & 0.03          & 0.01          & 0.40          & 0.32          & 0.15          & 0.14          \\
            \hline
            \textbf{RelEx}         & \textbf{0.94}                           & \textbf{0.95}                         & \textbf{0.96} & \textbf{0.97} & \textbf{0.78} & \textbf{0.66} & \textbf{0.60} & \textbf{0.56} \\
            \hline
        \end{tabular}
    } \label{tab_del_pre_score}
\end{table}

\begin{figure}[t]
    \begin{center}
        \includegraphics[width=\linewidth]{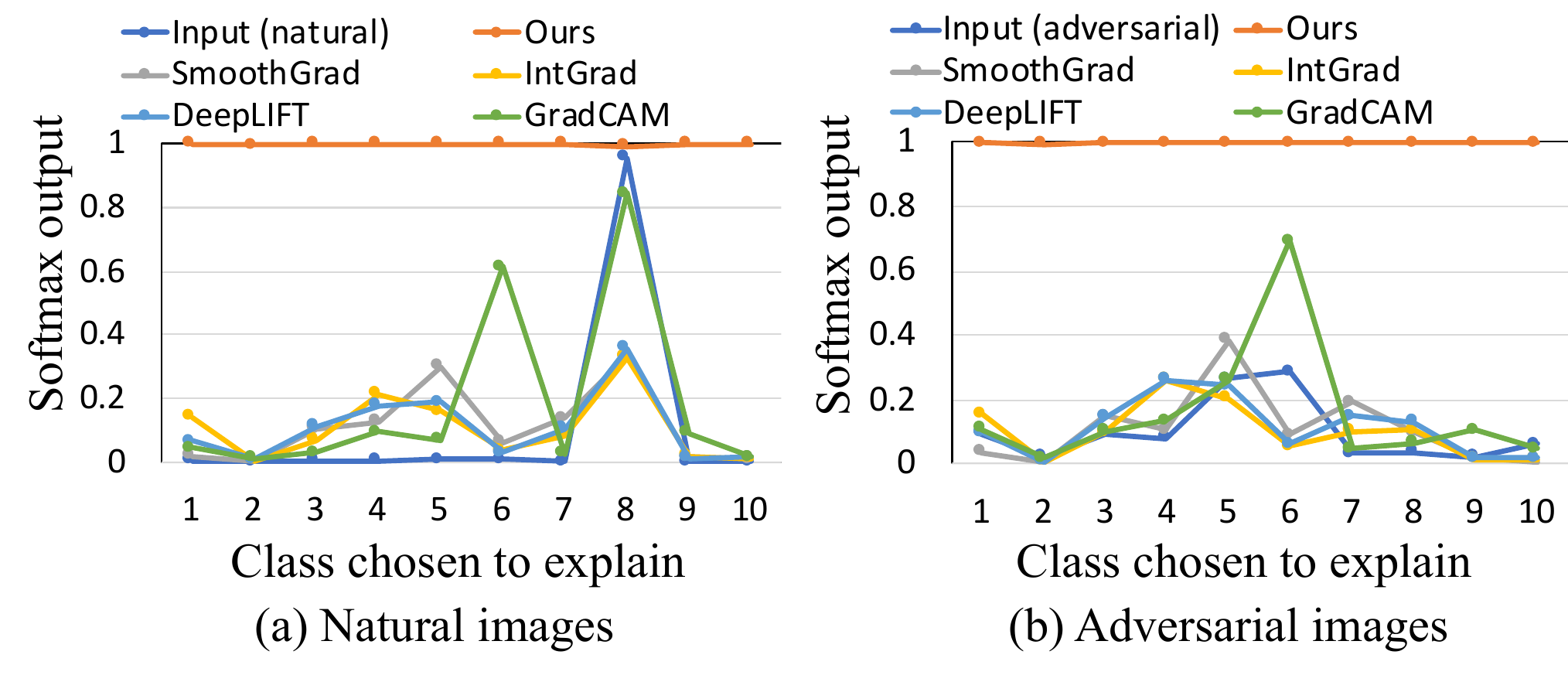}
    \end{center}
    \vspace*{-5mm}
    \caption{
        \textbf{Comparison of explanations for arbitrarily chosen classes on CIFAR-10.}
        (a) Plots show the softmax scores of explanations corresponding to all 10 classes for given natural images of \emph{class 8}.
        (b) Plots correspond to adversarial counterparts.
    }
    \vspace*{-2mm}
    \label{fig_cifar10}
\end{figure}

\textbf{Metrics for feature relevance evaluation.}
We use two metrics to evaluate the pixel-level relevancy for a given saliency map: \emph{deletion} and \emph{preservation}.
Deletion quantifies the accuracy of finding the smallest susceptive region, whereas preservation corresponds to the smallest evidential region, discussed in Section~\ref{subsec_optimization}.
See \supp~\ref{supp_feature_relevance_metric} for the details on the metrics.
The sole use of deletion is discouraged because, for instance, two extreme cases of the accurate and completely wrong smallest susceptive regions may have an identical deletion score.
Therefore, similar to the $F_1$ score, we use the harmonic mean of the deletion and preservation, $R$, as a metric of feature relevance;
in particular, $ \frac{1}{R} = \frac{1}{2} \left( \frac{1}{P} + \frac{1}{1-D} \right)$, where $P$ and $D$ are the preservation and deletion scores.
We used $1-D$ because a lower score results in better deletion.

\textbf{Results of the feature relevancy.}
Table~\ref{tab_del_pre_score} indicates that \ours~outperforms other methods in terms of the pixel-level relevancy, $R$.
The results confirm the evaluations of the target class retrieval, and visually plausible saliency maps do not necessarily present true evidence for explaining model predictions.
See \supp~\ref{supp_feature_relevance_result} for more results.

\textbf{Extracting explanations conditioned on arbitrary classes.}
Finally, we investigate whether our method can extract explanations conditioned on arbitrary non-target classes for a given input.
We sampled 400 images annotated as \emph{class 8} from the test set of CIFAR-10.
Then, we drew explanations of classes from the samples and their adversaries on natural ResNet-18 with the \linfnorm~of perturbation set to 8 out of 255~\cite{tsipras2018robustness}, and compared the softmax scores of our explanations to those of the existing methods.
Figure~\ref{fig_cifar10} illustrates that the scores of our explanations are close to 1 consistently through all classes on both the natural samples and their adversarial examples, outperforming others.
The results indicate that \ours~can extract explanations as long as relevant evidence exists in the input.
The explanations represent the specified classes faithfully, effectively excluding information on irrelevant classes.
We provide more results in \supp~\ref{supp_cifar10_results}.

%

\section{Conclusion}

\black{We introduced a reliable explanation of neural networks that} requires consistency on model outputs and the corresponding saliency maps along with neighboring data points.
The proposed method, \ours, addresses the concern by interpreting the model explanations via a locally smooth landscape with respect to the loss function of the model output.
Our analysis demonstrated that the smoothness in the landscape improves as we reduce the \lonenorm~of a saliency map.
The experimental results demonstrated that the proposed method based on the analysis identifies features relevant to the target class retrieval against the strong white-box attacks.
We also demonstrated that causal evidence for model predictions does not always coincide with visually appealing saliency maps as in previous methods.

\section*{Acknowledgement}
This work was supported by a National Research Foundation of Korea (NRF) grant funded by the Korean government (MSIT) (2019R1F1A1061941).

    {\small
        \bibliographystyle{ieee_fullname}
        \bibliography{manuscript}

\begin{thebibliography}{10}\itemsep=-1pt

\bibitem{bach2015pixel}
Sebastian Bach, Alexander Binder, Gr{\'e}goire Montavon, Frederick Klauschen,
  Klaus-Robert M{\"u}ller, and Wojciech Samek.
\newblock On pixel-wise explanations for non-linear classifier decisions by
  layer-wise relevance propagation.
\newblock {\em PloS one}, 10(7):e0130140, 2015.

\bibitem{chang2018explaining}
Chun-Hao Chang, Elliot Creager, Anna Goldenberg, and David Duvenaud.
\newblock Explaining image classifiers by counterfactual generation.
\newblock In {\em International Conference on Learning Representations}, 2018.

\bibitem{dabkowski2017real}
Piotr Dabkowski and Yarin Gal.
\newblock Real time image saliency for black box classifiers.
\newblock In {\em Advances in Neural Information Processing Systems}, pages
  6967--6976, 2017.

\bibitem{deng2009imagenet}
Jia Deng, Wei Dong, Richard Socher, Li-Jia Li, Kai Li, and Li Fei-Fei.
\newblock Imagenet: A large-scale hierarchical image database.
\newblock In {\em 2009 IEEE conference on computer vision and pattern
  recognition}, pages 248--255. Ieee, 2009.

\bibitem{dombrowski2019explanations}
Ann-Kathrin Dombrowski, Maximillian Alber, Christopher Anders, Marcel
  Ackermann, Klaus-Robert M{\"u}ller, and Pan Kessel.
\newblock Explanations can be manipulated and geometry is to blame.
\newblock In {\em Advances in Neural Information Processing Systems}, pages
  13589--13600, 2019.

\bibitem{du2018towards}
Mengnan Du, Ninghao Liu, Qingquan Song, and Xia Hu.
\newblock Towards explanation of dnn-based prediction with guided feature
  inversion.
\newblock In {\em Proceedings of the 24th ACM SIGKDD International Conference
  on Knowledge Discovery \& Data Mining}, pages 1358--1367, 2018.

\bibitem{fawzi2018adversarial}
Alhussein Fawzi, Hamza Fawzi, and Omar Fawzi.
\newblock Adversarial vulnerability for any classifier.
\newblock In {\em Advances in neural information processing systems}, pages
  1178--1187, 2018.

\bibitem{fawzi2018empirical}
Alhussein Fawzi, Seyed-Mohsen Moosavi-Dezfooli, Pascal Frossard, and Stefano
  Soatto.
\newblock Empirical study of the topology and geometry of deep networks.
\newblock In {\em Proceedings of the IEEE Conference on Computer Vision and
  Pattern Recognition}, pages 3762--3770, 2018.

\bibitem{fong2017interpretable}
Ruth~C Fong and Andrea Vedaldi.
\newblock Interpretable explanations of black boxes by meaningful perturbation.
\newblock In {\em Proceedings of the IEEE International Conference on Computer
  Vision}, pages 3429--3437, 2017.

\bibitem{ghorbani2019interpretation}
Amirata Ghorbani, Abubakar Abid, and James Zou.
\newblock Interpretation of neural networks is fragile.
\newblock In {\em Proceedings of the AAAI Conference on Artificial
  Intelligence}, volume~33, pages 3681--3688, 2019.

\bibitem{gilmer2018adversarial}
Justin Gilmer, Luke Metz, Fartash Faghri, Samuel~S Schoenholz, Maithra Raghu,
  Martin Wattenberg, and Ian Goodfellow.
\newblock Adversarial spheres.
\newblock {\em arXiv preprint arXiv:1801.02774}, 2018.

\bibitem{goodfellow2014explaining}
Ian~J Goodfellow, Jonathon Shlens, and Christian Szegedy.
\newblock Explaining and harnessing adversarial examples.
\newblock {\em arXiv preprint arXiv:1412.6572}, 2014.

\bibitem{he2016deep}
Kaiming He, Xiangyu Zhang, Shaoqing Ren, and Jian Sun.
\newblock Deep residual learning for image recognition.
\newblock In {\em Proceedings of the IEEE conference on computer vision and
  pattern recognition}, pages 770--778, 2016.

\bibitem{heo2019fooling}
Juyeon Heo, Sunghwan Joo, and Taesup Moon.
\newblock Fooling neural network interpretations via adversarial model
  manipulation.
\newblock In {\em Advances in Neural Information Processing Systems}, pages
  2925--2936, 2019.

\bibitem{isola2017image}
Phillip Isola, Jun-Yan Zhu, Tinghui Zhou, and Alexei~A Efros.
\newblock Image-to-image translation with conditional adversarial networks.
\newblock In {\em Proceedings of the IEEE conference on computer vision and
  pattern recognition}, pages 1125--1134, 2017.

\bibitem{kapishnikov2019xrai}
Andrei Kapishnikov, Tolga Bolukbasi, Fernanda Vi{\'e}gas, and Michael Terry.
\newblock Xrai: Better attributions through regions.
\newblock In {\em Proceedings of the IEEE International Conference on Computer
  Vision}, pages 4948--4957, 2019.

\bibitem{madry2018towards}
Aleksander Madry, Aleksandar Makelov, Ludwig Schmidt, Dimitris Tsipras, and
  Adrian Vladu.
\newblock Towards deep learning models resistant to adversarial attacks.
\newblock In {\em International Conference on Learning Representations}, 2018.

\bibitem{melis2018towards}
David~Alvarez Melis and Tommi Jaakkola.
\newblock Towards robust interpretability with self-explaining neural networks.
\newblock In {\em Advances in Neural Information Processing Systems}, pages
  7775--7784, 2018.

\bibitem{moosavi2019robustness}
Seyed-Mohsen Moosavi-Dezfooli, Alhussein Fawzi, Jonathan Uesato, and Pascal
  Frossard.
\newblock Robustness via curvature regularization, and vice versa.
\newblock In {\em Proceedings of the IEEE Conference on Computer Vision and
  Pattern Recognition}, pages 9078--9086, 2019.

\bibitem{Petsiuk2018rise}
Vitali Petsiuk, Abir Das, and Kate Saenko.
\newblock Rise: Randomized input sampling for explanation of black-box models.
\newblock In {\em Proceedings of the British Machine Vision Conference (BMVC)},
  2018.

\bibitem{qin2019adversarial}
Chongli Qin, James Martens, Sven Gowal, Dilip Krishnan, Krishnamurthy
  Dvijotham, Alhussein Fawzi, Soham De, Robert Stanforth, and Pushmeet Kohli.
\newblock Adversarial robustness through local linearization.
\newblock In {\em Advances in Neural Information Processing Systems}, pages
  13847--13856, 2019.

\bibitem{rebuffi2020there}
Sylvestre-Alvise Rebuffi, Ruth Fong, Xu Ji, and Andrea Vedaldi.
\newblock There and back again: Revisiting backpropagation saliency methods.
\newblock In {\em Proceedings of the IEEE/CVF Conference on Computer Vision and
  Pattern Recognition}, pages 8839--8848, 2020.

\bibitem{ross2017improving}
Andrew~Slavin Ross and Finale Doshi-Velez.
\newblock Improving the adversarial robustness and interpretability of deep
  neural networks by regularizing their input gradients.
\newblock {\em arXiv preprint arXiv:1711.09404}, 2017.

\bibitem{schmidt2018adversarially}
Ludwig Schmidt, Shibani Santurkar, Dimitris Tsipras, Kunal Talwar, and
  Aleksander Madry.
\newblock Adversarially robust generalization requires more data.
\newblock In {\em Advances in Neural Information Processing Systems}, pages
  5014--5026, 2018.

\bibitem{selvaraju2017grad}
Ramprasaath~R Selvaraju, Michael Cogswell, Abhishek Das, Ramakrishna Vedantam,
  Devi Parikh, and Dhruv Batra.
\newblock Grad-cam: Visual explanations from deep networks via gradient-based
  localization.
\newblock In {\em Proceedings of the IEEE international conference on computer
  vision}, pages 618--626, 2017.

\bibitem{shaham2015understanding}
Uri Shaham, Yutaro Yamada, and Sahand Negahban.
\newblock Understanding adversarial training: Increasing local stability of
  neural nets through robust optimization.
\newblock {\em arXiv preprint arXiv:1511.05432}, 2015.

\bibitem{shrikumar2017learning}
Avanti Shrikumar, Peyton Greenside, and Anshul Kundaje.
\newblock Learning important features through propagating activation
  differences.
\newblock In {\em International Conference on Machine Learning}, pages
  3145--3153, 2017.

\bibitem{simonyan2013deep}
Karen Simonyan, Andrea Vedaldi, and Andrew Zisserman.
\newblock Deep inside convolutional networks: Visualising image classification
  models and saliency maps.
\newblock {\em arXiv preprint arXiv:1312.6034}, 2013.

\bibitem{slack2020fooling}
Dylan Slack, Sophie Hilgard, Emily Jia, Sameer Singh, and Himabindu Lakkaraju.
\newblock Fooling lime and shap: Adversarial attacks on post hoc explanation
  methods.
\newblock In {\em Proceedings of the AAAI/ACM Conference on AI, Ethics, and
  Society}, pages 180--186, 2020.

\bibitem{smilkov2017smoothgrad}
Daniel Smilkov, Nikhil Thorat, Been Kim, Fernanda Vi{\'e}gas, and Martin
  Wattenberg.
\newblock Smoothgrad: removing noise by adding noise.
\newblock {\em arXiv preprint arXiv:1706.03825}, 2017.

\bibitem{spearman1987proof}
Charles Spearman.
\newblock The proof and measurement of association between two things.
\newblock {\em The American journal of psychology}, 100(3/4):441--471, 1987.

\bibitem{springenberg2014striving}
Jost~Tobias Springenberg, Alexey Dosovitskiy, Thomas Brox, and Martin
  Riedmiller.
\newblock Striving for simplicity: The all convolutional net.
\newblock {\em arXiv preprint arXiv:1412.6806}, 2014.

\bibitem{subramanya2019fooling}
Akshayvarun Subramanya, Vipin Pillai, and Hamed Pirsiavash.
\newblock Fooling network interpretation in image classification.
\newblock In {\em Proceedings of the IEEE International Conference on Computer
  Vision}, pages 2020--2029, 2019.

\bibitem{sundararajan2017axiomatic}
Mukund Sundararajan, Ankur Taly, and Qiqi Yan.
\newblock Axiomatic attribution for deep networks.
\newblock In {\em Proceedings of the 34th International Conference on Machine
  Learning-Volume 70}, pages 3319--3328, 2017.

\bibitem{szegedy2013intriguing}
Christian Szegedy, Wojciech Zaremba, Ilya Sutskever, Joan Bruna, Dumitru Erhan,
  Ian Goodfellow, and Rob Fergus.
\newblock Intriguing properties of neural networks.
\newblock {\em arXiv preprint arXiv:1312.6199}, 2013.

\bibitem{tsipras2018robustness}
Dimitris Tsipras, Shibani Santurkar, Logan Engstrom, Alexander Turner, and
  Aleksander Madry.
\newblock Robustness may be at odds with accuracy.
\newblock In {\em International Conference on Learning Representations}, 2018.

\bibitem{wagner2019interpretable}
Jorg Wagner, Jan~Mathias Kohler, Tobias Gindele, Leon Hetzel, Jakob~Thaddaus
  Wiedemer, and Sven Behnke.
\newblock Interpretable and fine-grained visual explanations for convolutional
  neural networks.
\newblock In {\em Proceedings of the IEEE Conference on Computer Vision and
  Pattern Recognition}, pages 9097--9107, 2019.

\bibitem{wang2018high}
Ting-Chun Wang, Ming-Yu Liu, Jun-Yan Zhu, Andrew Tao, Jan Kautz, and Bryan
  Catanzaro.
\newblock High-resolution image synthesis and semantic manipulation with
  conditional gans.
\newblock In {\em Proceedings of the IEEE conference on computer vision and
  pattern recognition}, pages 8798--8807, 2018.

\bibitem{zhang2019theoretically}
Hongyang Zhang, Yaodong Yu, Jiantao Jiao, Eric Xing, Laurent El~Ghaoui, and
  Michael Jordan.
\newblock Theoretically principled trade-off between robustness and accuracy.
\newblock In {\em International Conference on Machine Learning}, pages
  7472--7482, 2019.

\bibitem{zhu2017unpaired}
Jun-Yan Zhu, Taesung Park, Phillip Isola, and Alexei~A Efros.
\newblock Unpaired image-to-image translation using cycle-consistent
  adversarial networks.
\newblock In {\em Proceedings of the IEEE international conference on computer
  vision}, pages 2223--2232, 2017.

\end{thebibliography}


\begin{thebibliography}{1}\itemsep=-1pt

\bibitem{dombrowski2019explanations}
Ann-Kathrin Dombrowski, Maximillian Alber, Christopher Anders, Marcel
  Ackermann, Klaus-Robert M{\"u}ller, and Pan Kessel.
\newblock Explanations can be manipulated and geometry is to blame.
\newblock In {\em Advances in Neural Information Processing Systems}, pages
  13589--13600, 2019.

\bibitem{ghorbani2019interpretation}
Amirata Ghorbani, Abubakar Abid, and James Zou.
\newblock Interpretation of neural networks is fragile.
\newblock In {\em Proceedings of the AAAI Conference on Artificial
  Intelligence}, volume~33, pages 3681--3688, 2019.

\bibitem{madry2018towards}
Aleksander Madry, Aleksandar Makelov, Ludwig Schmidt, Dimitris Tsipras, and
  Adrian Vladu.
\newblock Towards deep learning models resistant to adversarial attacks.
\newblock In {\em International Conference on Learning Representations}, 2018.

\bibitem{papernot2018cleverhans}
Nicolas Papernot, Fartash Faghri, Nicholas Carlini, Ian Goodfellow, Reuben
  Feinman, Alexey Kurakin, Cihang Xie, Yash Sharma, Tom Brown, Aurko Roy,
  Alexander Matyasko, Vahid Behzadan, Karen Hambardzumyan, Zhishuai Zhang,
  Yi-Lin Juang, Zhi Li, Ryan Sheatsley, Abhibhav Garg, Jonathan Uesato, Willi
  Gierke, Yinpeng Dong, David Berthelot, Paul Hendricks, Jonas Rauber, and
  Rujun Long.
\newblock Technical report on the cleverhans v2.1.0 adversarial examples
  library.
\newblock {\em arXiv preprint arXiv:1610.00768}, 2018.

\bibitem{smilkov2017smoothgrad}
Daniel Smilkov, Nikhil Thorat, Been Kim, Fernanda Vi{\'e}gas, and Martin
  Wattenberg.
\newblock Smoothgrad: removing noise by adding noise.
\newblock {\em arXiv preprint arXiv:1706.03825}, 2017.

\bibitem{spearman1987proof}
Charles Spearman.
\newblock The proof and measurement of association between two things.
\newblock {\em The American journal of psychology}, 100(3/4):441--471, 1987.

\bibitem{sundararajan2017axiomatic}
Mukund Sundararajan, Ankur Taly, and Qiqi Yan.
\newblock Axiomatic attribution for deep networks.
\newblock In {\em Proceedings of the 34th International Conference on Machine
  Learning-Volume 70}, pages 3319--3328, 2017.

\bibitem{wagner2019interpretable}
Jorg Wagner, Jan~Mathias Kohler, Tobias Gindele, Leon Hetzel, Jakob~Thaddaus
  Wiedemer, and Sven Behnke.
\newblock Interpretable and fine-grained visual explanations for convolutional
  neural networks.
\newblock In {\em Proceedings of the IEEE Conference on Computer Vision and
  Pattern Recognition}, pages 9097--9107, 2019.

\end{thebibliography}
    }

\end{document}


\title{
    Building Reliable Explanations of Unreliable Neural Networks: Locally Smoothing Perspective of Model Interpretation\\
    (Supplementary Materials)
}

\author{
    {Dohun Lim}
    \qquad{Hyeonseok Lee}
    \qquad{Sungchan Kim}\\
Division of Computer Science and Engineering, Jeonbuk National University, Korea\\
{\tt\small \{imdohun75,hslee0390,s.kim\}@jbnu.ac.kr}
}

\maketitle

\tableofcontents

\section{The Proofs} \label{supp_proofs}

\begin{theorem}[\textbf{Local explanations with respect to label consistency}]
    Let $\gamma = \alpha \cdot v$ where $||\gamma|| = \alpha$ and $||v|| = 1$.
    Let $\mathcal{D} = \{x_i\}$ be a set of data samples where $||x_i - x_0|| \leqq \epsilon$.
    For a given saliency map $m$ calculated from $x_0$, it holds that
    \begin{equation}
        \alpha \geqq \frac{c}{ || m || _1} \cdot \frac{ 2 } { || -g(x_0 + \alpha v) + g(x_0) ||  + 2 || g(x_0) ||  }
    \end{equation}
    where
    \begin{align}
        g(x) = -\nabla L (x, m) = \nabla \log f(m\odot x).
    \end{align}
    It also holds that the loss function in Eq.~\eqref{eq_loss_robust} with respect to $x_0 + \gamma$ is upper-bounded as follows:
    \begin{equation}
        \begin{split}
            &L(x_0 + \gamma, m ) \leqq \\
            &\alpha ||m||_1 \left( \frac{||-g(x_0 + \alpha v) + g(x_0)||}{2}  + || g(x_0) || \right).
        \end{split}
    \end{equation}
\end{theorem}

\begin{proof}
    We begin with the definition of Hessian of the loss function $L(x_0, m)$ with respect to input $x$ at $x_0$ to learn a saliency $m$ as
    \begin{align}
         & H = \frac{ \nabla L ( x_0+\gamma , m ) - \nabla L( x_0 , m ) }{\gamma} =  \frac{ \nabla L ( x_0+ \alpha v , m ) - \nabla L( x_0 , m ) }{ \alpha v }
    \end{align}
    Then,
    \begin{align} \label{eq_hessian}
         & H \gamma = H \alpha v \approx \nabla L ( x_0 + \alpha v , m ) - \nabla L ( x_0 , m )
    \end{align}
    By substituting $L(m \odot x) = - \log f ( m \odot x )$ into Eq.~\eqref{eq_hessian}, we have
    \begin{align}
        H \gamma & = - \nabla \log f (m \odot (x_0 + \alpha v)) +  \nabla \log f(m \odot x_0)                                                       \\
                 & = - \left . \frac{\partial \log f( m \odot x + \alpha m \odot v)}{\partial ( m \odot x) } \odot m \right \vert_{x_0 +  \alpha v}
        + \left . \frac{\partial \log f( m \odot x)}{\partial ( m \odot x) } \odot m \right \vert_{x_0} \label{eq_Hr}
    \end{align}
    We rewrite $g(x)$ in Eq.~\eqref{eq_g_x} as
    \begin{equation}
        g(x) = \frac{\partial \log f( m \odot x)}{\partial ( m \odot x) }.
    \end{equation}
    Then, Eq.~\eqref{eq_Hr} is given by
    \begin{equation} \label{eq_hessian_2}
        H \gamma = \left \{ -g(x_0 + \alpha v) + g(x_0) \right \} \odot m
    \end{equation}
    Using the Cauchy–Schwarz inequality, for the constraint of Eq.~\eqref{eq_gamma_argmin_2}, the following holds
    \begin{align}\label{eq_CS_ineq_1}
        \nabla L (x_0 , m)^T \gamma + \frac{1}{2} \gamma^T H \gamma & \leqq || \nabla L ( x_0  , m)^T || \cdot ||\gamma|| + \frac{1}{2} ||\gamma|| \cdot || H \gamma || \\
    \end{align}
    By integrating Eq.~\eqref{eq_hessian_2} into Eq.~\eqref{eq_CS_ineq_1},
    \begin{align}\label{eq_CS_ineq_2}
        \nabla L (x_0 , m)^T \gamma + \frac{1}{2} \gamma^T H \gamma & \leqq \alpha ||g(x_0)|| \cdot ||m|| + \frac{\alpha}{2} || -g(x_0 + \alpha v) + g(x_0) || \cdot ||m||
    \end{align}
    Using $||m|| \leqq ||m||_1$, it holds that
    \begin{align}\label{eq_CS_ineq_3}
         & \nabla L (x_0 , m)^T \gamma + \frac{1}{2} \gamma^T H \gamma \leqq \alpha ||g(x_0)||\cdot||m||_1 + \frac{\alpha}{2}|| -g(x_0 + \alpha v) + g(x_0) || \cdot ||m||_1.
    \end{align}
    Combining Eq.~\eqref{eq_gamma_argmin_2} and Eq.~\eqref{eq_CS_ineq_3} gives
    \begin{align} \label{eq_theorem_1_last_1}
         & \alpha ||g( x_0 )|| \cdot ||m||_1 + \frac{\alpha}{2}|| -g(x_0 + \alpha v) + g(x_0) || \cdot ||m||_1 \geqq c
    \end{align}
    By rearranging Eq.~\eqref{eq_theorem_1_last_1}, we reach Eq.~\eqref{eq_alpha_bound},
    \begin{align*}
        \alpha \geqq \frac{c}{ || m || _1} \cdot \frac{ 2 } { || -g(x_0 + \alpha v) + g(x_0) ||  + 2 || g(x_0) ||  }.
    \end{align*}
    Finally, combining Eq.~\eqref{eq_loss_robust} and Eq.~\eqref{eq_CS_ineq_3} gives Eq.~\eqref{eq_classification_loss_bound} as
    \begin{align*}
        L( x_0 + \gamma  , m) \leqq \alpha ||m||_1 \left( \frac{1}{2} || -g(x_0 + \alpha v) + g(x_0) || + || g(x_0) || \right)
    \end{align*}
\end{proof}

\begin{theorem}[\textbf{Local explanations with respect to saliency map consistency}]
    Let $\mathcal{D} = \{ x_i \}$ be the vicinity of the input data $x_0$ such that $||x_i - x_0|| \leqq \epsilon$ where $\epsilon$ being a small positive number.
    Then, distance between the gradients of explanations of $x_i$ and $x_0$ is lower-bounded as follows:
    \begin{equation}
        || \nabla L(x_i, m ) - \nabla L (x_0, m ) || \leqq || m ||_1 \cdot || -g(x_0 + \alpha v) + g(x_0) ||.
    \end{equation}
\end{theorem}

\begin{proof}
    We begin from the following.
    \begin{equation}
        \nabla L(x , m) = - \log f ( m \odot x ) = - \frac{\partial \log f( m \odot x)}{\partial ( m \odot x) } \odot m
    \end{equation}
    Based on the requirement on the robustness of a saliency map as described in Assumption~\ref{assumption_saliency_robustness}, we assume that the gradient of a data point $x_i$ can be written using the first-order Taylor expansion at $x_0$, which is given by
    \begin{equation}
        \nabla L(x_i , m) = \nabla L(x_0 + \gamma , m) \approx \nabla L(x_0 , m) + H \gamma.
    \end{equation}
    Then, we consider distance between the gradients of $x_i$ and $x_0$ as
    \begin{equation}
        || \nabla L(x_i , m) - \nabla L (x_0 , m) || = || H \gamma ||.
    \end{equation}
    Similar to the steps we took corresponding to Eq.~(\ref{eq_CS_ineq_1}) in the proof of Theorem~\ref{theorem_cls_robust}, the following holds
    \begin{equation*}
        || \nabla L( x_i  , m) - \nabla L ( x_0  , m) || \leqq || m ||_1 \cdot || -g(x_0 + \alpha v) + g(x_0) || .
    \end{equation*}
\end{proof}

%
%
\section{Experimental Setups}

\subsection{Implementation Details} \label{supp_implementation_details}
For the objective function in Eq.~\eqref{eq_full_optimization}, given an input image $x_0 = \{ x_{0,i} \} \in \mathbbm{R}^d $ as a vector of $d$ pixles, we created a batch of 100 neighboring data points with respect to $x_0$, $\mathcal{D}$, by adding random noise following the normal distribution, $N(0, \sigma)$, to each pixel $x_{0,i}$, where $\sigma$ is a standard deviation and $\sigma = 0.1 \times \left( \max (x_{0,i}) - \min (x_{0,i} ) \right)$.

A saliency map $m$ is initialized following the uniform distribution on the interval $[0, 0.01]$.
We set the parameters of the objective function as $\lambda_1 = 0.0001$ and $\lambda_2 = 1.0$, respectively.
We solve Eq.~\eqref{eq_full_optimization} using the stochastic gradient descent (SGD) for $50$ epochs with the learning rate set to 0.001.
We denote this baseline by \emph{$Ours (50)$} in Figure~\ref{fig_class_robustness_untargeted_attack_resnset50}.
We used normalized gradient in the optimization process when applying to the SGD.
We found that the normalization performed better in terms of the quality of saliency maps and the stability of the optimization.\\

\textbf{Computation time.}
The current implementation of the proposed method takes about 17 seconds to solve the optimization with the baseline setting for an image from ImageNet on ResNet-50 using a single RTX 2080 Ti GPU.\\

\textbf{Post-processing of saliency maps.}
It is required to post-process a saliency map to construct an explanation for a given image by multiplying a saliency map and the image.
For the proposed method, we use a saliency as a result of the optimization in Eq.~\eqref{eq_full_optimization} directly with no additional processing of the saliency map.
However, we applied two strategies differently according to the previous methods.
The first strategy is taken from~\cite{smilkov2017smoothgrad} and applies to SmGrad, IntGrad, SimGrad, and DeepLIFT.
This strategy takes the expectation of absolute values of 3-channels, RGB, for each pixel in a saliency map.
Then, a final saliency is given by normalizing the expectation for each pixel to the \nth{99} percentile of high value.

The second strategy applies to GradCAM and RT-Sal, where a saliency map consists of single-channel pixel-wise values.
Let $g = \{ g_i \}$ be a saliency map that is a result of GradCAM or RT-Sal and, thus, not normalized.
Then $m = \{ m_i \}$, which is a normalized counterpart that we use to evaluate the explanations, is given by
\begin{equation*}
    m_i = \frac{ g_i - \min(g_i)}{ \max(g_i) - \min(g_i)}.
\end{equation*}

\subsection{Generation of Adversarial Examples} \label{supp_adversary_generation}

\textbf{Implementing the untargeted attack.}
We used CleverHans~\cite{papernot2018cleverhans} to implement the PGD-based untargeted attack.
We primarily follow the procedure in~\cite{madry2018towards}.
In particular, we applied 40 iterations of the PGD attack when generating adversarial examples.
Because the ResNet model in the experiments takes pre-process images where pixel values are scaled to $[-2.117, +2.639]$, we also changed the step size and the range of $\ell_{\infty}$-norm of perturbation for the PGD attack in the literature to 0.01 and \{0.07, 0.1, 0.3, 1, 2, 4, 8\} in this study, respectively.
\\

\textbf{Implementing the targeted attacks.}
We used the codes released by~\cite{ghorbani2019interpretation} and~\cite{dombrowski2019explanations} to implement the unstructured and the structured attacks respectively.

For the structured attack, we applied the attack for 1500 iterations to the natural ResNet-50 with a learning rate of 0.0002.
We set two prefactor values of $10^{11}$ and $10^6$ for the terms in the objective function of the attack, which correspond to a saliency map and the accuracy loss, respectively~\cite{dombrowski2019explanations}.

In the case of the unstructured attack, we applied a \emph{top-$k$ fooling} that aims to generate a false saliency map where the top-$k$ feature importance of the saliency map of an input image is reduced as much as possible.
We set $k$ to 1000 in the experiments.
We applied the attack for 300 iterations.
Due to the different preprocessing setting of the model as in the case of the untargeted attack, we rescaled the step size and the perturbation distance used in~\cite{ghorbani2019interpretation} accordingly.

%
%
\section{Additional Experimental Results}

\subsection{Comparison of Perturbation Distance for the Adversarial Attacks} \label{supp_targeted_perturbation_distance}

\begin{table} [h]
    \centering
    \caption{Difference of adversarial examples compared to their clean counterparts from ImageNet when applying the PGD-based untargeted attack to the natural ResNet-50 by varying $\ell_{\infty}$-norm of perturbation.} \label{supp_tab_perturbation_distance_untargeted_attack}
    \begin{tabular}{|c|c|c|c|c|c|c|c|}
        \hline
        $\ell_{\infty}$-norm of perturbation & 0.07    & 0.1     & 0.3     & 1        & 2        & 4        & 8        \\
        \hline\hline
        Difference in $\ell_2$-norm          & 17.8596 & 24.3589 & 67.5621 & 215.2923 & 405.0308 & 680.0658 & 861.8575 \\
        \hline
    \end{tabular}
\end{table}

\begin{table} [h]
    \centering
    \caption{Difference of adversarial examples compared to their clean counterparts from ImageNet when applying the targeted attacks to the natural ResNet-50 against each method.} \label{supp_tab_perturbation_distance_targeted_attack}
    \begin{tabular}{|c|c|c|c|c|c|}
        \hline
        Methods                              & \multicolumn{1}{l|}{Difference in $\ell_2$-norm} & IntGrad  & SmGrad  & RT-Sal  & GradCam  \\
        \hline\hline
        \multirow{3}{*}{Structured attack}   & min.                                             & 14.951   & 21.448  & 13.549  & 15.345   \\
        \cline{2-6}
                                             & avg.                                             & 628.346  & 628.28  & 628.104 & 628.123  \\
        \cline{2-6}
                                             & max.                                             & 1007.269 & 1005.27 & 1006.81 & 1006.802 \\
        \hline
        \multirow{3}{*}{Unstructured attack} & min.                                             & 8.609    & 3.301   & 3.377   & 3.02     \\
        \cline{2-6}
                                             & avg.                                             & 54.487   & 16.847  & 15.589  & 49.815   \\
        \cline{2-6}
                                             & max.                                             & 170.175  & 22.06   & 21.812  & 169.387  \\
        \hline
    \end{tabular}
\end{table}

Table~\ref{supp_tab_perturbation_distance_targeted_attack} and Table~\ref{supp_tab_perturbation_distance_untargeted_attack} show that the perturbation distance caused by the targeted attacks is comparable to the case where the $\ell_{\infty}$-norm of perturbation is 4.
Because the proposed method is robust in such a level of perturbation against the untargeted attack as shown in the paper (Figure~\ref{fig_class_robustness_untargeted_attack_resnset50}), we assume that it is impractical to create adversarial images by applying the targeted attacks against our method.

%
%
\subsection{Results on Similarity of Saliency Maps} \label{supp_saliency_map_similarity}

In addition to \emph{Spearman's rank-order correlation}~\cite{spearman1987proof}, we provide the results of another metric, a \emph{top-$k$ intersection}, to evaluate the similarity of saliency maps.
The top-$k$ intersection measures the size of the intersection of the $k$-most important features between a clean image and its adversary~\cite{ghorbani2019interpretation}.

Figure~\ref{fig_supp_topk} shows the results of top-$1000$ intersection, confirming that spatial similarity is unrelated to the fidelity of saliency maps to the model predictions.

\begin{figure}[h]
    \begin{center}
        \includegraphics[width=0.8\linewidth]{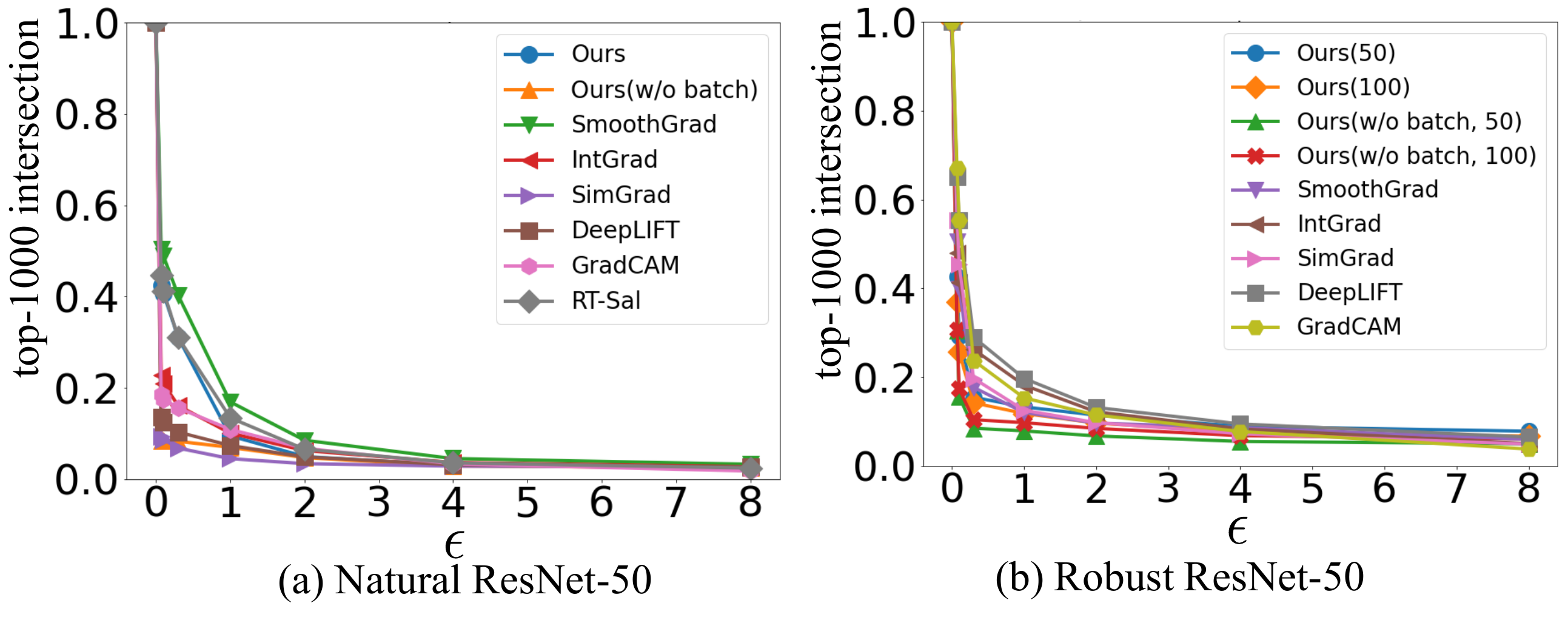}
    \end{center}
    \caption{
        Similarity of the saliency maps of the adversarial examples in the \textbf{top-1000 intersection} against the \textbf{untargeted attack}.
    }
    \label{fig_supp_topk}
\end{figure}

%
%
\subsection{Metrics for Feature Relevance Evaluation} \label{supp_feature_relevance_metric}

Deletion and preservation are metrics to evaluate the fidelity of a given saliency map with respect to the softmax score of the corresponding class.\\

\textbf{Deletion.}
A deletion score, which is also known as \emph{pixel flipping}, is calculated as follows.
First, we sort pixels of the input image according in descending order of their corresponding values in the saliency map.
We apply flipping all pixels to zero in the sorted order, creating a plot that represent a target class score of a given input, which is an explanation in this study.
We measure area-under-cover (AUC) of the plot as the deletion score of the saliency map.
In general, a significant drop should appear as early as possible with a saliency map of high fidelity.
Thus, better deletion results in a low deletion score.
See plots on left for each of the methods in Figure~\ref{fig_supp_del_pre_natural_resnet} and Figure~\ref{fig_supp_del_pre_robust_resnet}.\\

\textbf{Preservation.}
The measurement of a preservation score is similar to the case of a deletion score but pixels of an input image are sorted in the ascending order with respect to the saliency map.
Thus, the drop should be as late as possible, meaning that irrelevant pixels are removed earlier than relevant ones.
A high score indicates better preservation as opposed to the case of deletion.
See plots on the right for each of the methods in Figure~\ref{fig_supp_del_pre_natural_resnet}  and Figure~\ref{fig_supp_del_pre_robust_resnet}.\\

\begin{table}[H]
    \centering
    \caption{Deletion and preservation scores of the methods on ImageNet when applied to the natural ResNet-50.}
    \label{supp_tab_del_pre_clean_images}
    \begin{tabular}{|c|c|c|}
        \hline
        Method                               & \begin{tabular}[c]{@{}c@{}}Deletion\\(lower is better)\end{tabular} & \begin{tabular}[c]{@{}c@{}}Preservation\\(higher is better)\end{tabular} \\
        \hline\hline
        SimGrad                              & 0.1336                     & 0.2519                     \\
        \hline
        GradCAM                              & 0.1232                     & 0.5647                     \\
        \hline
        SmGrad                               & 0.0800                     & 0.3845                     \\
        \hline
        IntGrad                              & 0.0907                     & 0.3650                     \\
        \hline
        DeepLIFT                             & 0.0980                     & 0.3570                     \\
        \hline
        FGVis~\cite{wagner2019interpretable} & 0.0644                     & -                          \\
        \hline
        \textbf{RelEx (proposed)}            & \textbf{0.0567}            & \textbf{0.4093}            \\
        \hline
    \end{tabular}
\end{table}

Table~\ref{supp_tab_del_pre_clean_images} shows the deletion and preservation scores by applying each method to \emph{clean images of the ImageNet validation set} on the natural ResNet-50.
The score of FGVis~\cite{wagner2019interpretable} is taken from their paper due to the unavailability of implementation, to the best of our knowledge, which was the state-of-the-art deletion score.
Our method outperformed FGVis, achieving a new state-of-the-art performance in both deletion and preservation scores.

%
%
\subsection{Results on the Targeted Attacks} \label{supp_targeted_attack_result}

Figure~\ref{fig_supp_untargeted_attack} depicts the additional results of the unstructured attack against three more methods, SmGrad, IntGrad, and GradCAM.

\begin{figure}[H]
    \begin{center}
        \includegraphics[width=0.7\linewidth]{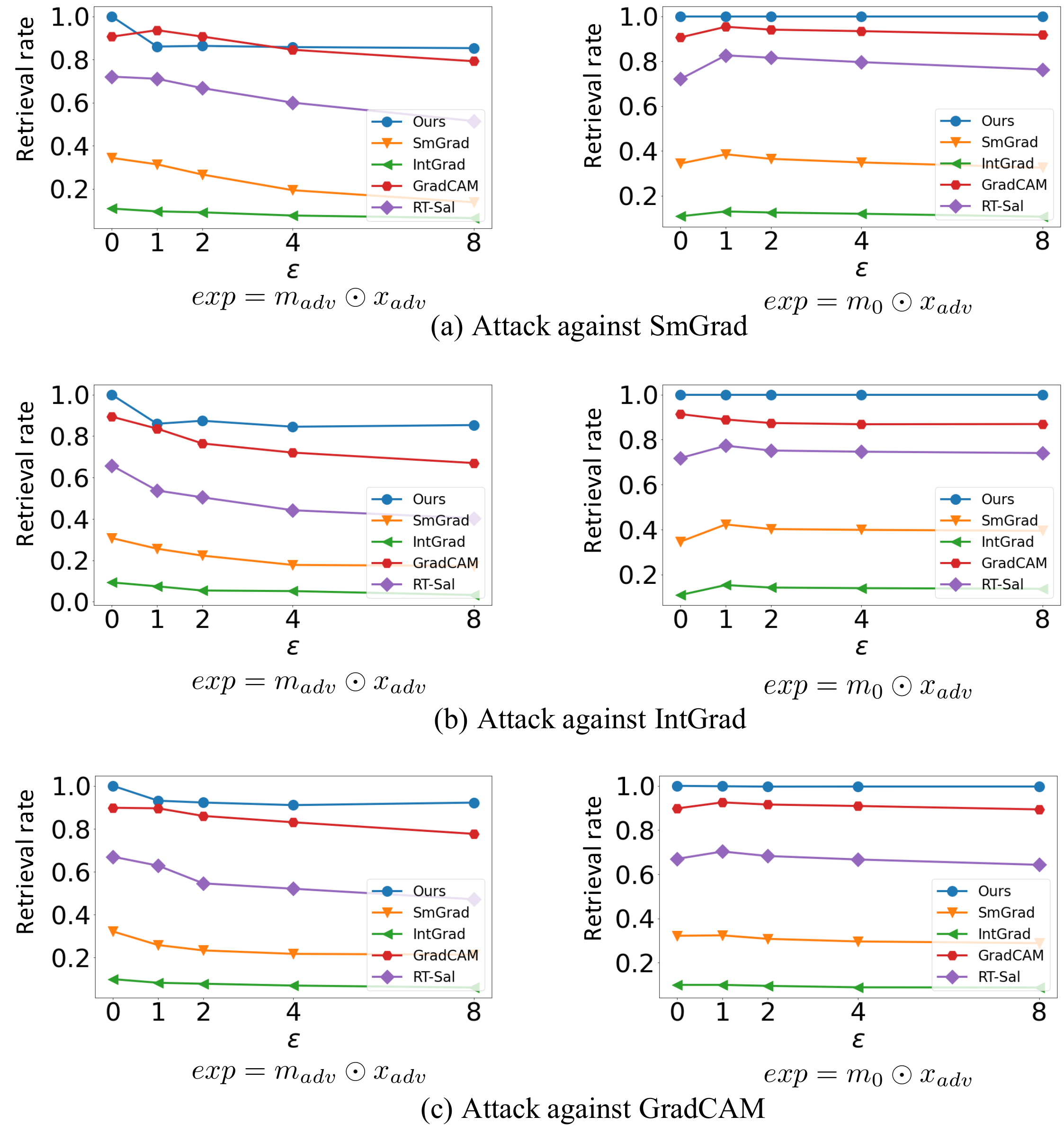}
    \end{center}
    \caption{
        Target class retrieval performance against the \textbf{unstructured attacks} on the \textbf{natural ResNet-50} for an explanation presented below each plot.
        Plots correspond to the adversarial images against (a) SmGrad~\cite{smilkov2017smoothgrad}, (b) IntGrad~\cite{sundararajan2017axiomatic}, and (c) GradCAM~\cite{smilkov2017smoothgrad}.
    }
    \label{fig_supp_untargeted_attack}
\end{figure}

\sectionbreak

%
%
\subsection{Results on Feature Relevancy of Saliency Maps} \label{supp_feature_relevance_result}

We provide the deletion and the preservation plots with respect to each of the explanation methods against the untargeted attack.

\begin{figure}[H]
    \begin{center}
        \includegraphics[width=\linewidth]{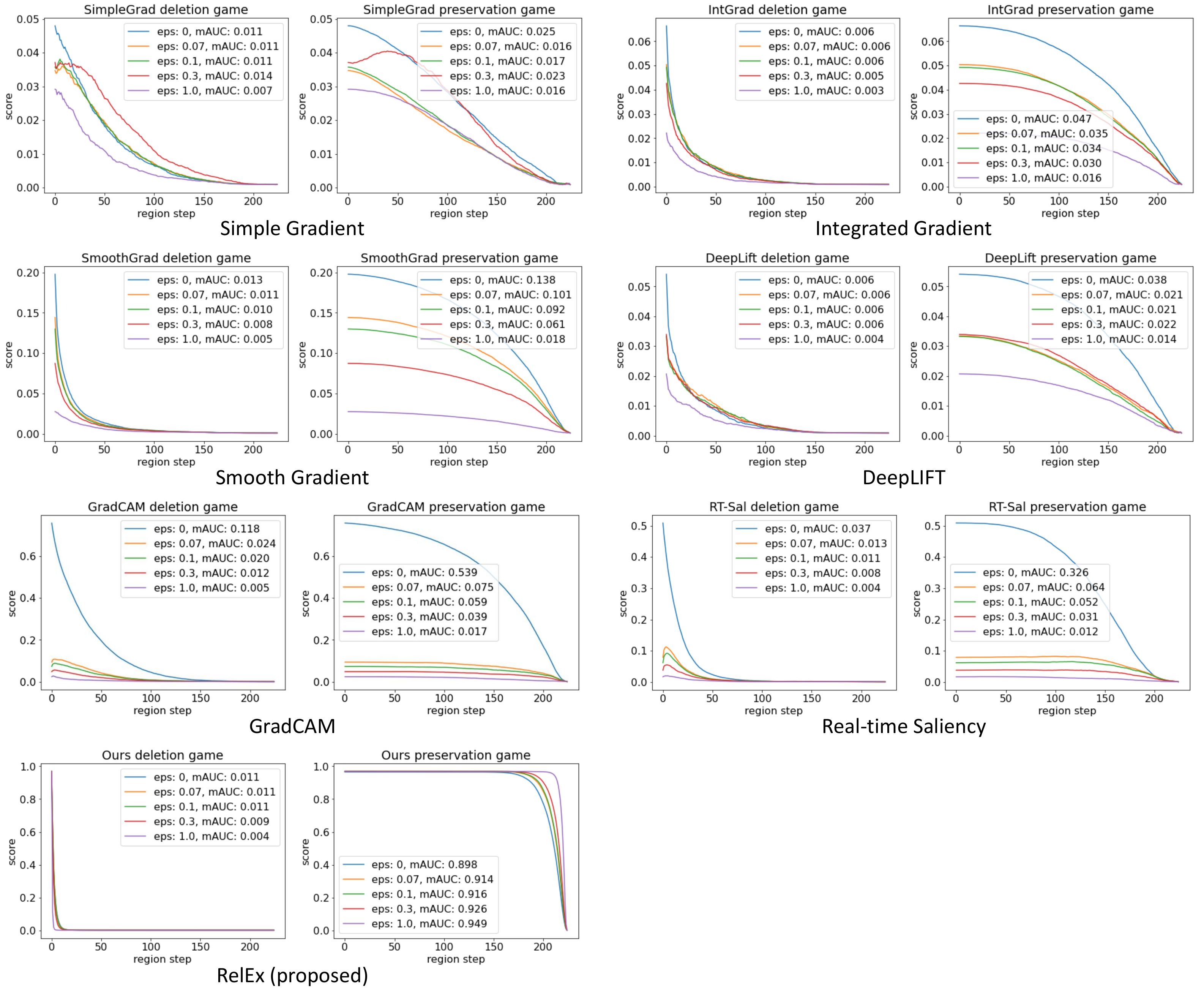}
    \end{center}
    \caption{
        The relevancy of saliency maps for each method in terms of deletion and preservation scores on the \textbf{natural ResNet-50} against the \textbf{untargeted attack}.
    }
    \label{fig_supp_del_pre_natural_resnet}
\end{figure}

\begin{figure}[H]
    \begin{center}
        \includegraphics[width=\linewidth]{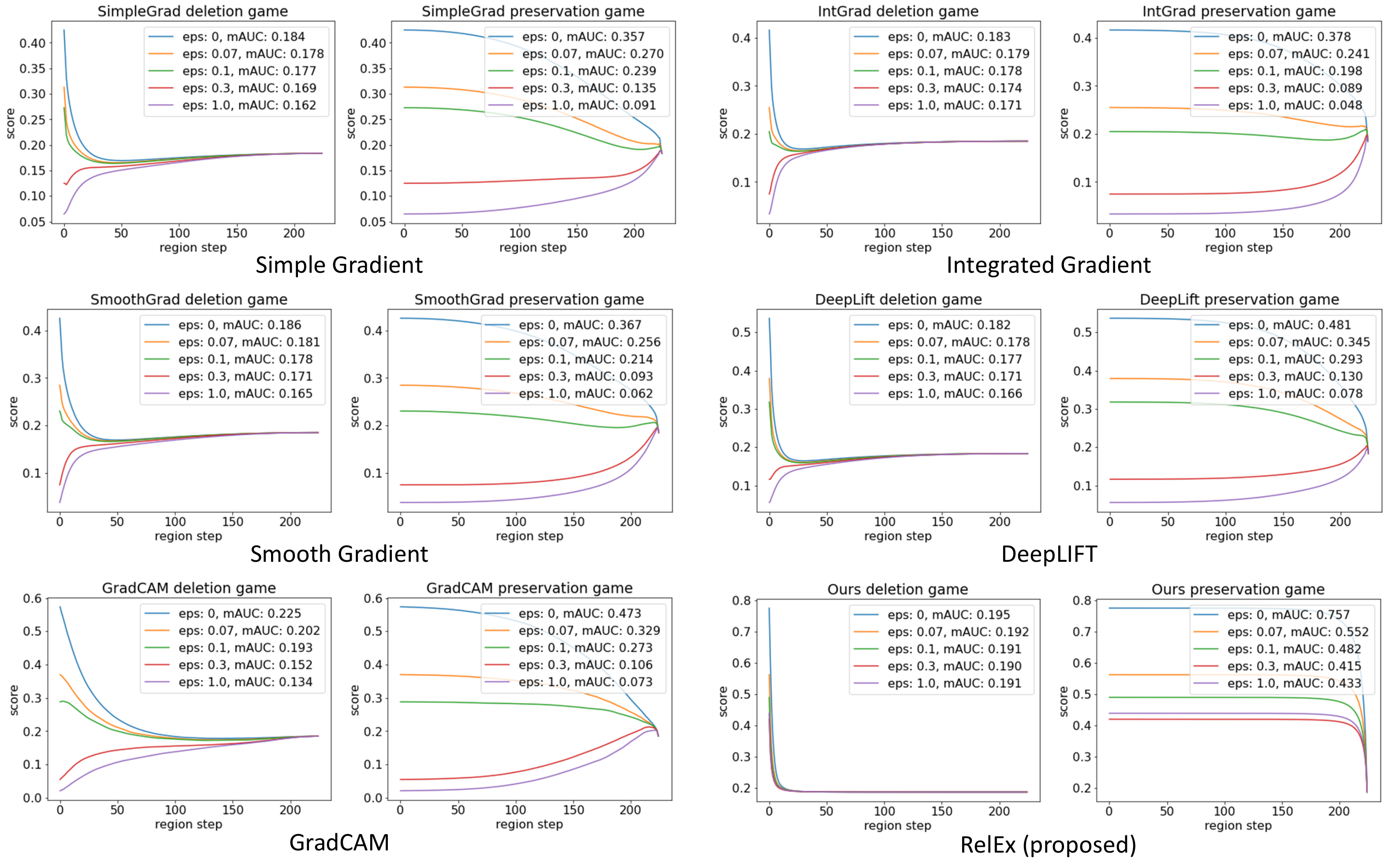}
    \end{center}
    \caption{
        The relevancy of saliency maps for each method in terms of deletion and preservation scores on the \textbf{robust ResNet-50} against the \textbf{untargeted attack.}
    }
    \label{fig_supp_del_pre_robust_resnet}
\end{figure}

\sectionbreak

%
%
\subsection{Extracting Explanations of Arbitrary Classes} \label{supp_cifar10_results}

Figure~\ref{figs_supp_cifar10} illustrates the additional results of explaining arbitrary classes on CIFAR-10.
The proposed method created explanations of three non-target classes, \emph{classes 1, 5}, and \emph{9}, that were applied to randomly selected images of the target \emph{class 7}.
The experiments were applied to both the clean images and their adversaries created by the PGD-based untargeted attack.

The plots show that the proposed method extracted the explanations of the chosen non-target class correctly.
The explanations faithfully contain the information of their corresponding classes.
In other words, almost no evidence on other classes were captures in the explanations.

\begin{figure}[H]
    \begin{center}
        \includegraphics[width=0.7\linewidth]{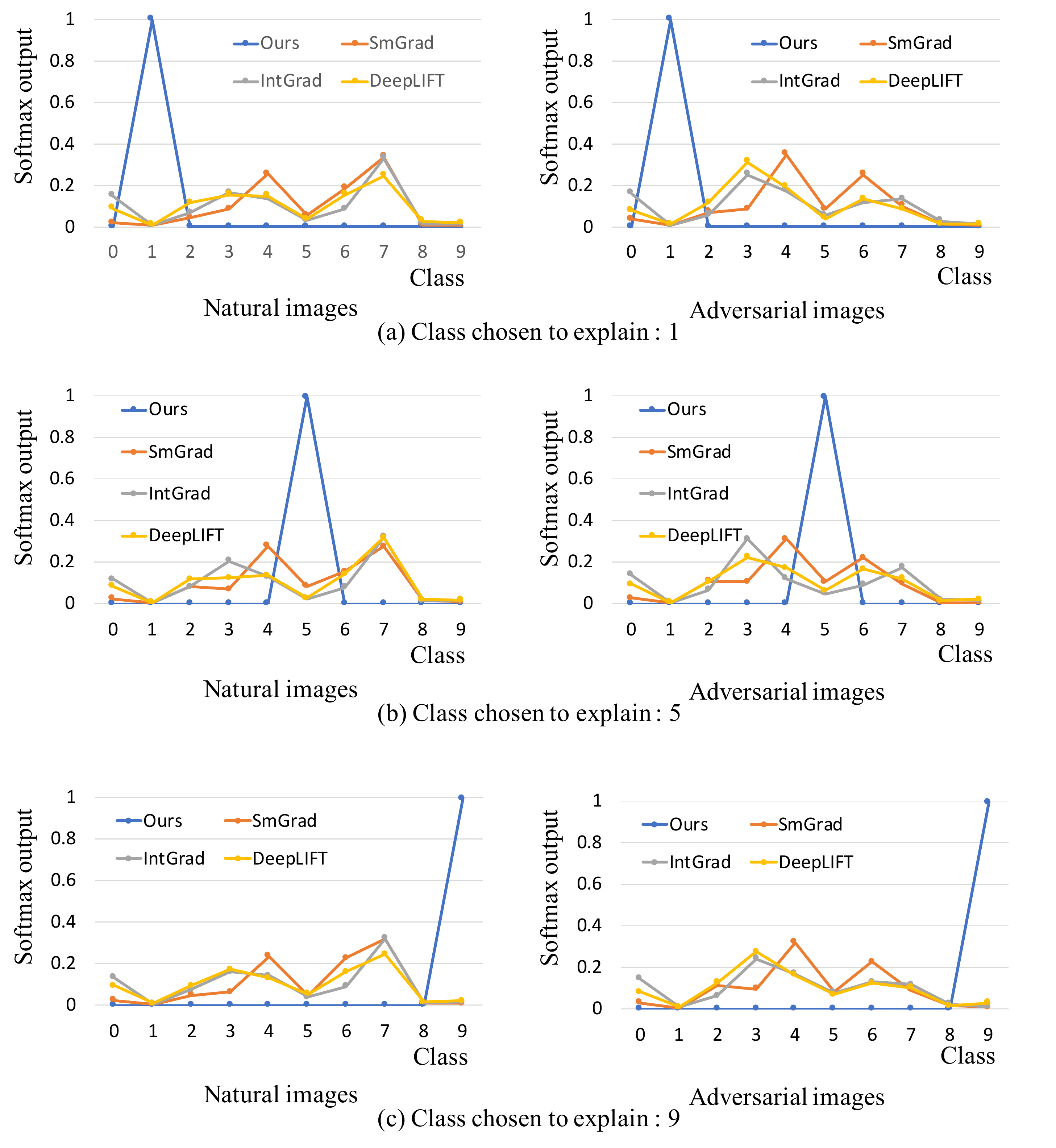}
    \end{center}
    \caption{
        Explaining non-target classes that are arbitrarily chosen.
        The horizontal axes of the plots enumerate all classes in the \textbf{CIFAR-10 dataset} on the \textbf{natural ResNet-18} and the vertical ones correspond to the softmax scores of all the classes for a given explanation of the arbitrary classes.
    }
    \label{figs_supp_cifar10}
\end{figure}

\sectionbreak

%
%
\subsection{Additional Qualitative Results}  \label{supp_qualitative_results}

We present additional qualitative results of saliency maps.
In the following figures, numbers below images represent the softmax scores of the target classes of the images.
Numbers below saliency maps represent the softmax scores of the target classes with respect to explanations corresponding to the saliency maps.
$\epsilon$ denotes $\ell_{\infty}$-norm of perturbation.
Saliency maps are best viewed zoomed-in on screen.

%
%
\subsubsection{The Results of the Untargeted Attack for the Natural ResNet-50}

\begin{figure}[H]
    \begin{center}
        \includegraphics[width=0.60\linewidth]{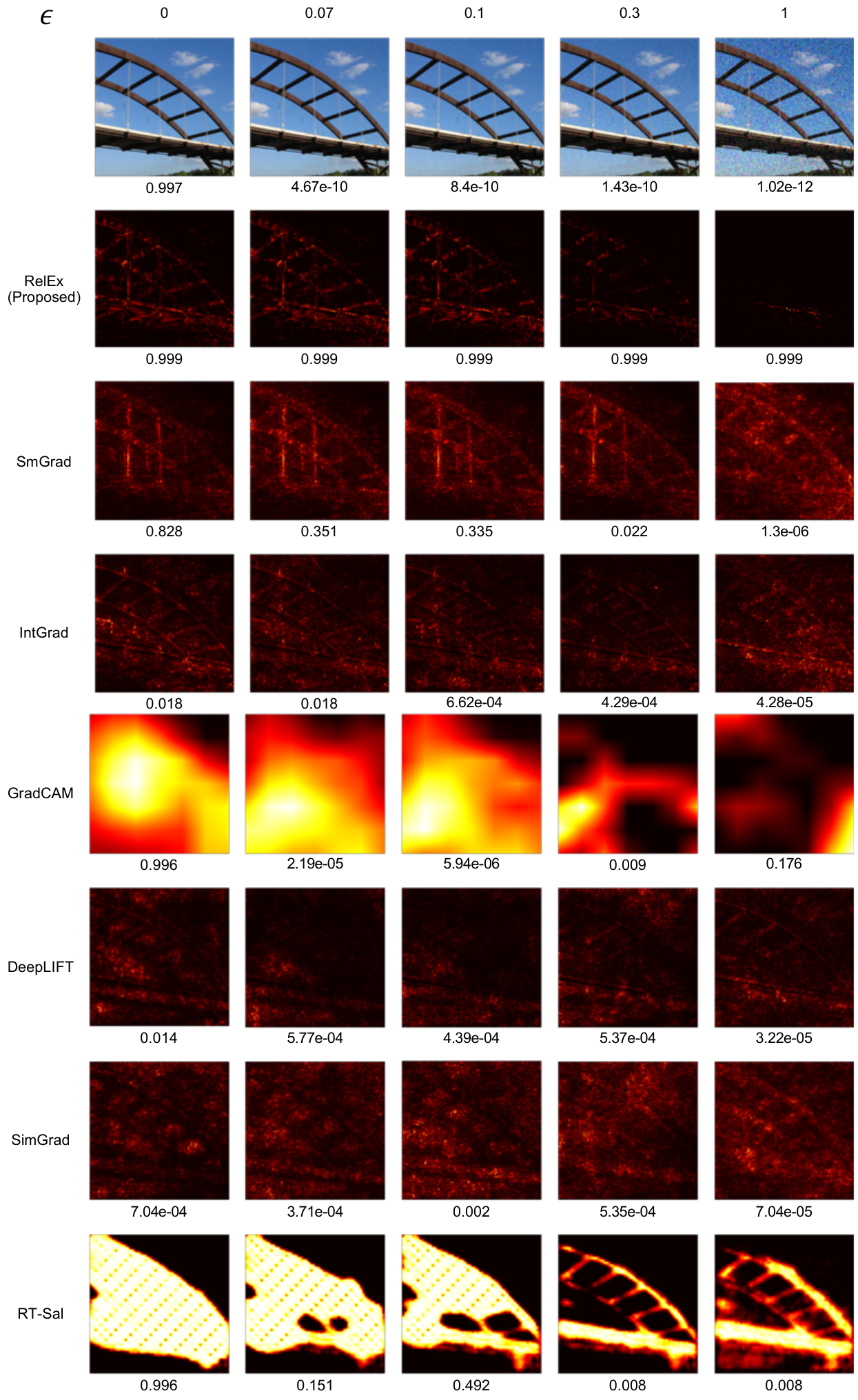}
        \caption{Qualitative results of the methods on the \textbf{natural ResNet-50} against the \textbf{untargeted attack}.} \label{fig_supp_untargeted_nat}
    \end{center}
\end{figure}

\begin{figure}[H]\ContinuedFloat
    \begin{center}
        \includegraphics[width=0.7\linewidth]{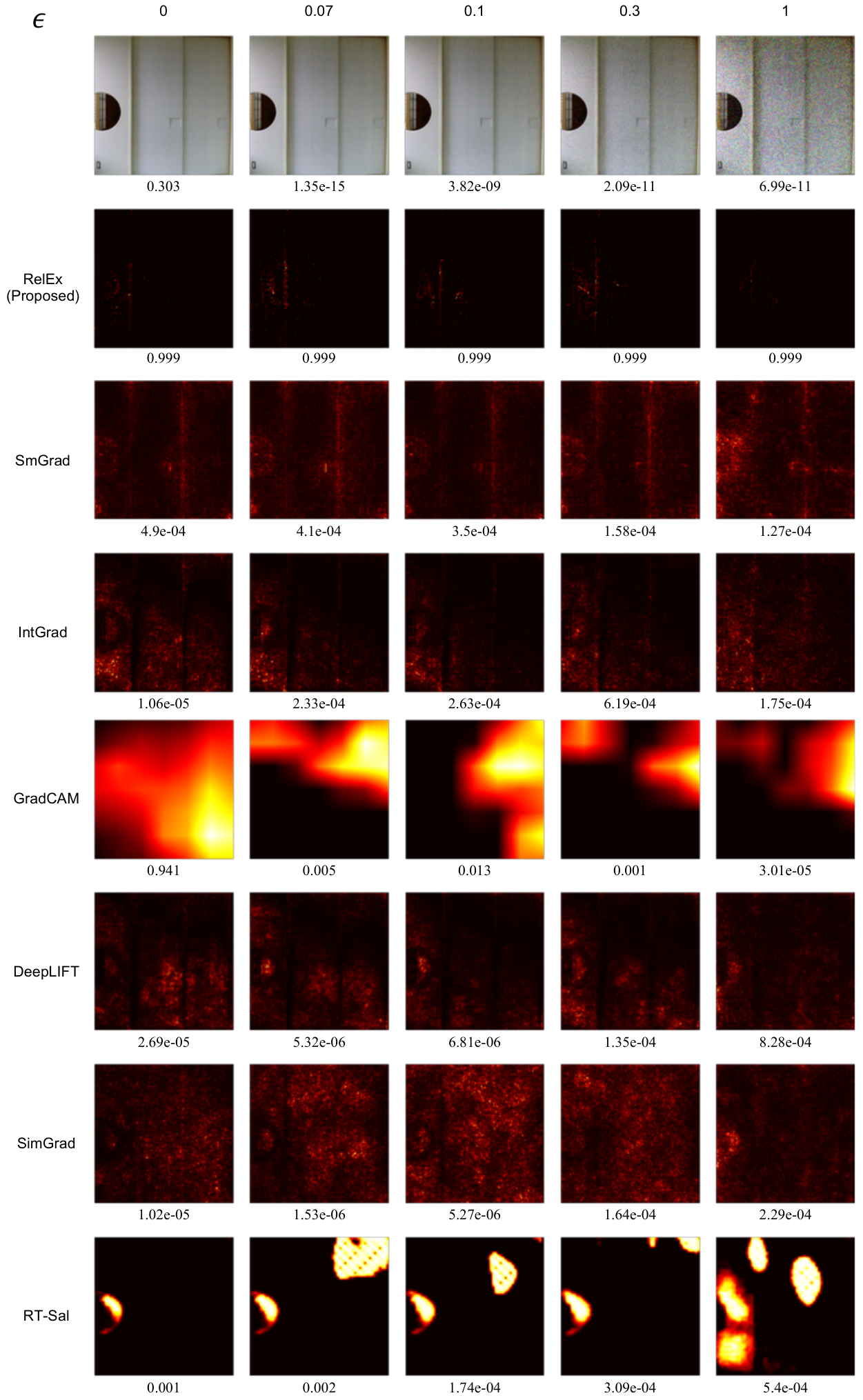}
        \caption{Qualitative results of the methods on the \textbf{natural ResNet-50} against the \textbf{untargeted attack} (continued).}
    \end{center}
\end{figure}

\begin{figure}[H]\ContinuedFloat
    \begin{center}
        \includegraphics[width=0.7\linewidth]{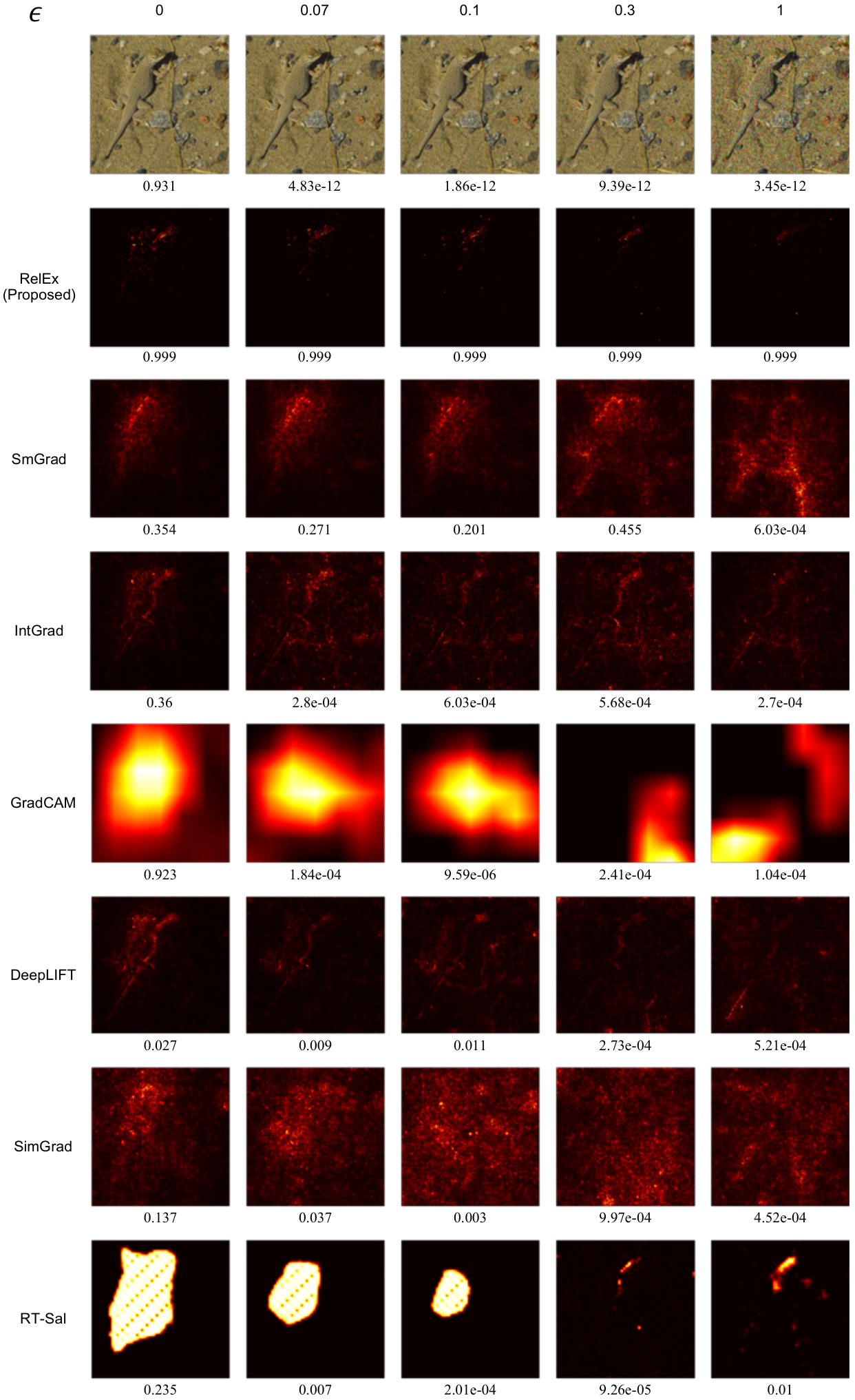}
        \caption{ Qualitative results of the methods on the \textbf{natural ResNet-50} against the \textbf{untargeted attack} (continued). }
    \end{center}
\end{figure}

\sectionbreak

%
%
\subsubsection{The Results of the Untargeted Attack against the Robust ResNet-50}

\begin{figure}[H]
    \begin{center}
        \includegraphics[width=0.7\linewidth]{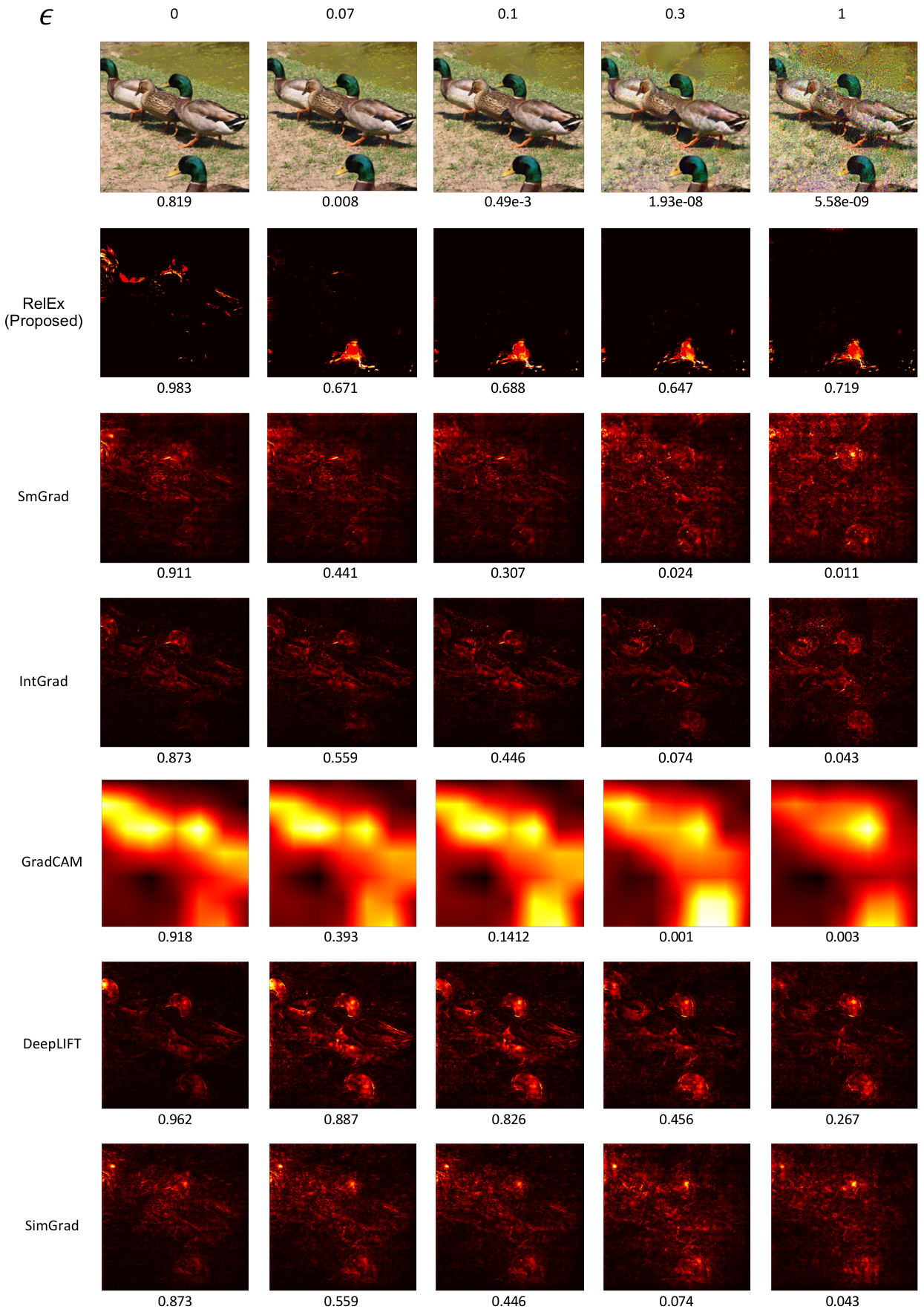}
        \caption{ Qualitative results of the methods on the \textbf{robust ResNet-50} against the \textbf{untargeted attack}. } \label{fig_supp_untargeted_robust}
    \end{center}
\end{figure}

\begin{figure}[H]\ContinuedFloat
    \begin{center}
        \includegraphics[width=0.7\linewidth]{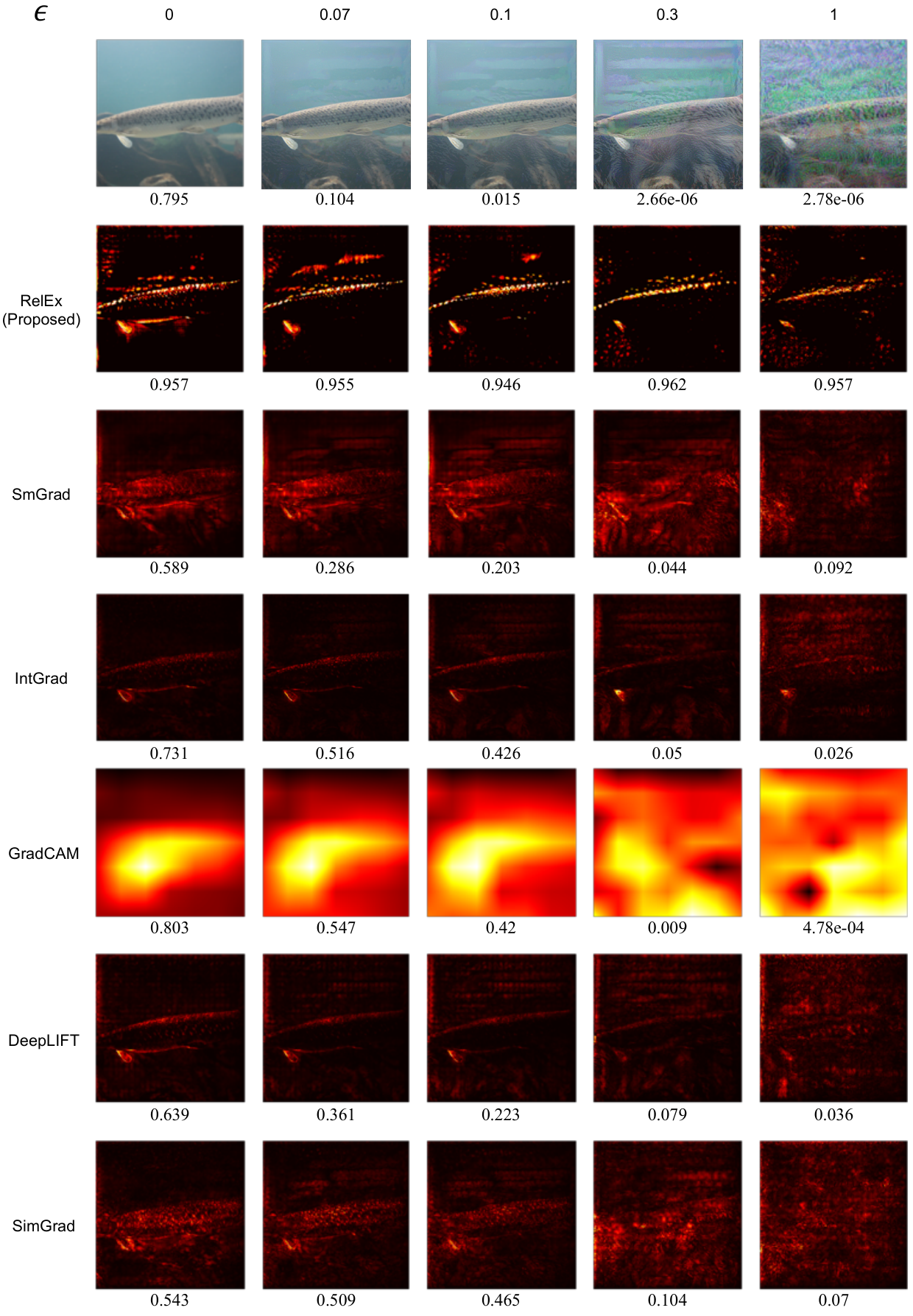}
        \caption{ Qualitative results of the methods on the \textbf{robust ResNet-50} against the \textbf{untargeted attack} (continued). }
    \end{center}
\end{figure}

\begin{figure}[H]\ContinuedFloat
    \begin{center}
        \includegraphics[width=0.7\linewidth]{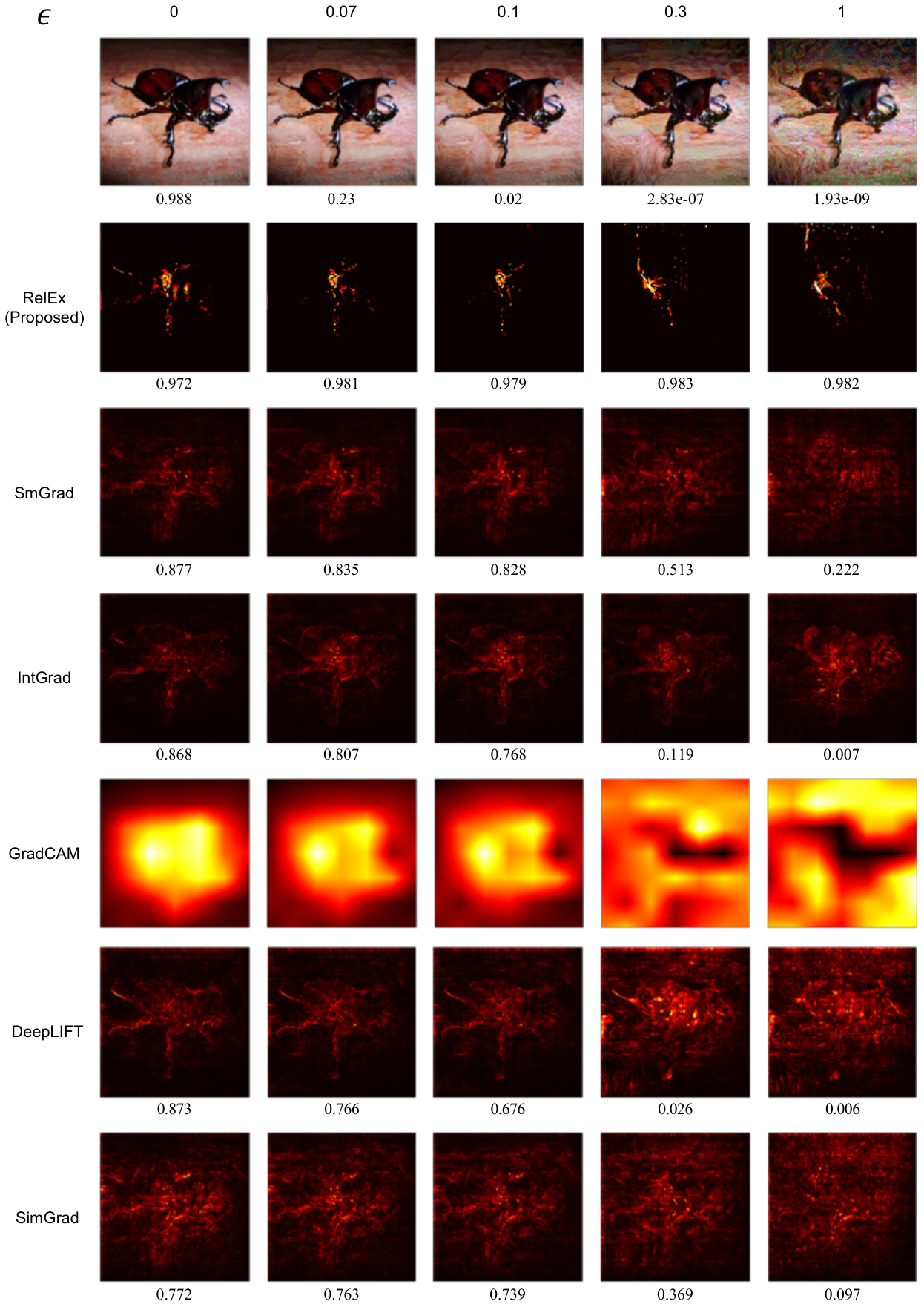}
        \caption{ Qualitative results of the methods on the \textbf{robust ResNet-50} against the \textbf{untargeted attack} (continued). }
    \end{center}
\end{figure}

\sectionbreak

\sectionbreak
%
%
\subsubsection{The Results of the Targeted Attacks on the Natural ResNet-50: Structured Attack}

The structure attack aims to change the saliency of an \emph{original} image to that of the \emph{target} image.
We use the "cat" picture as the target image for all examples below.
The first and the second row represent saliency maps of the target and the original images by each method, respectively.
We applied the structured attack to SmGrad, IntGrad, GradCAM, and RT-Sal (the red box), creating adversarial images for each method being attacked.
Then, we extracted explanations of the adversarial images against a method by using other methods (the blue box) as shown below.
Our method RelEx generates consistent saliency maps against the attacks to all the methods unlike other methods.

\begin{figure}[H]
    \begin{center}
        \includegraphics[width=\linewidth]{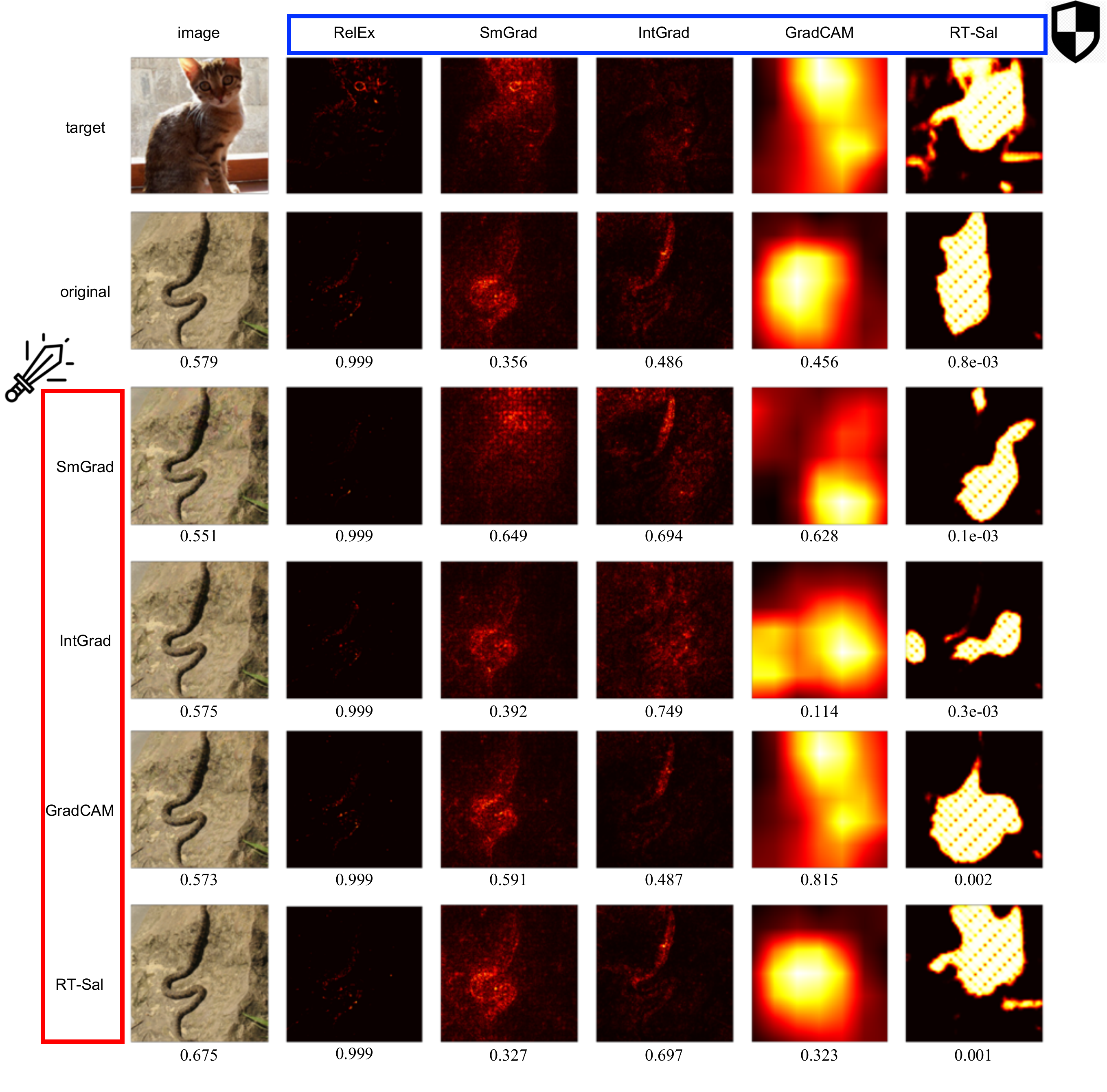}
        \caption{Qualitative results of the methods on the \textbf{natural ResNet-50} against the \textbf{structured attack}. } \label{fig_supp_structured_nat}
    \end{center}
\end{figure}

\begin{figure}[H]\ContinuedFloat
    \begin{center}
        \includegraphics[width=\linewidth]{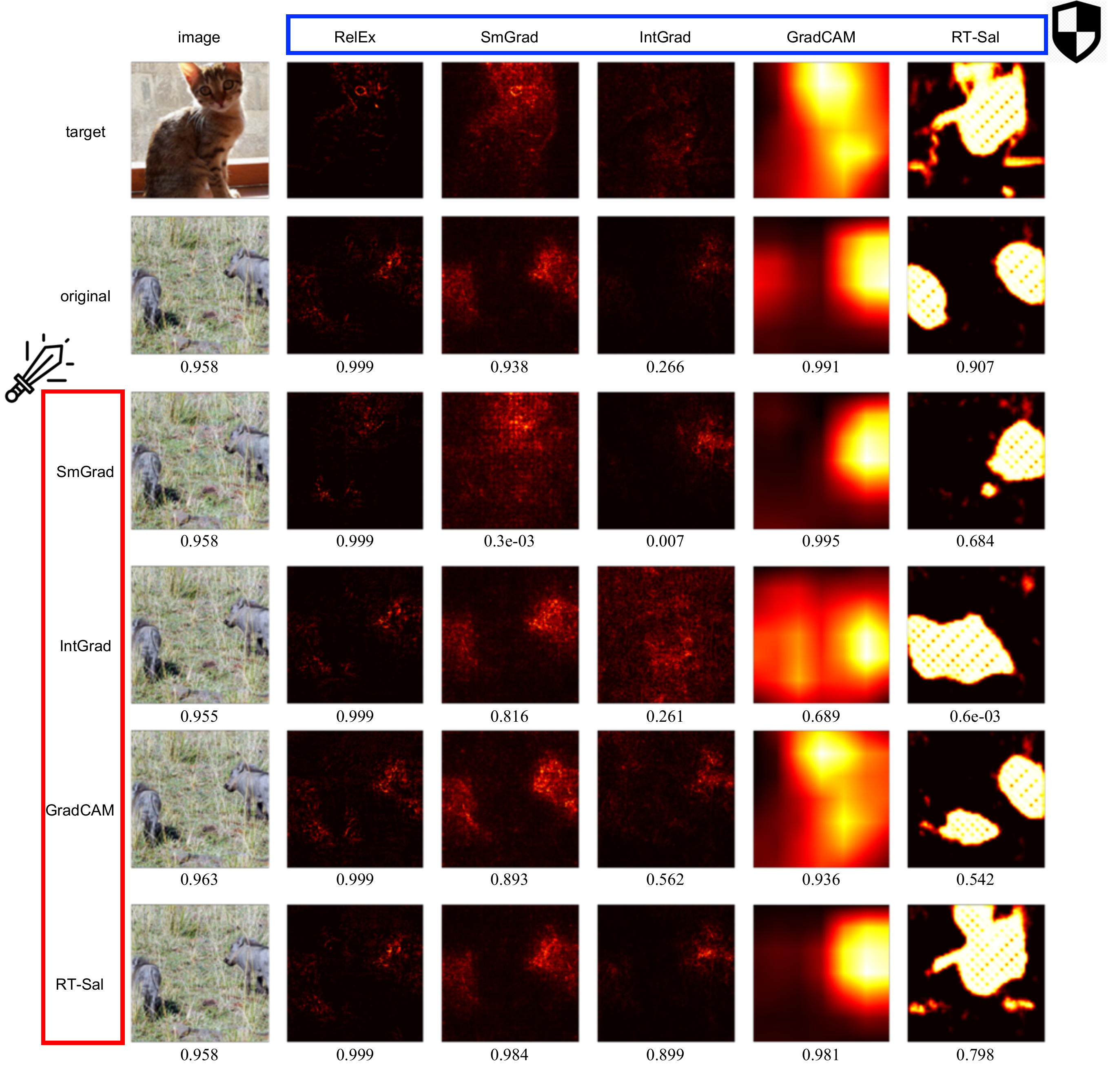}
        \caption{Qualitative results of the methods on the \textbf{natural ResNet-50} against the \textbf{structured attack} (continued). }
    \end{center}
\end{figure}

\begin{figure}[H]\ContinuedFloat
    \begin{center}
        \includegraphics[width=\linewidth]{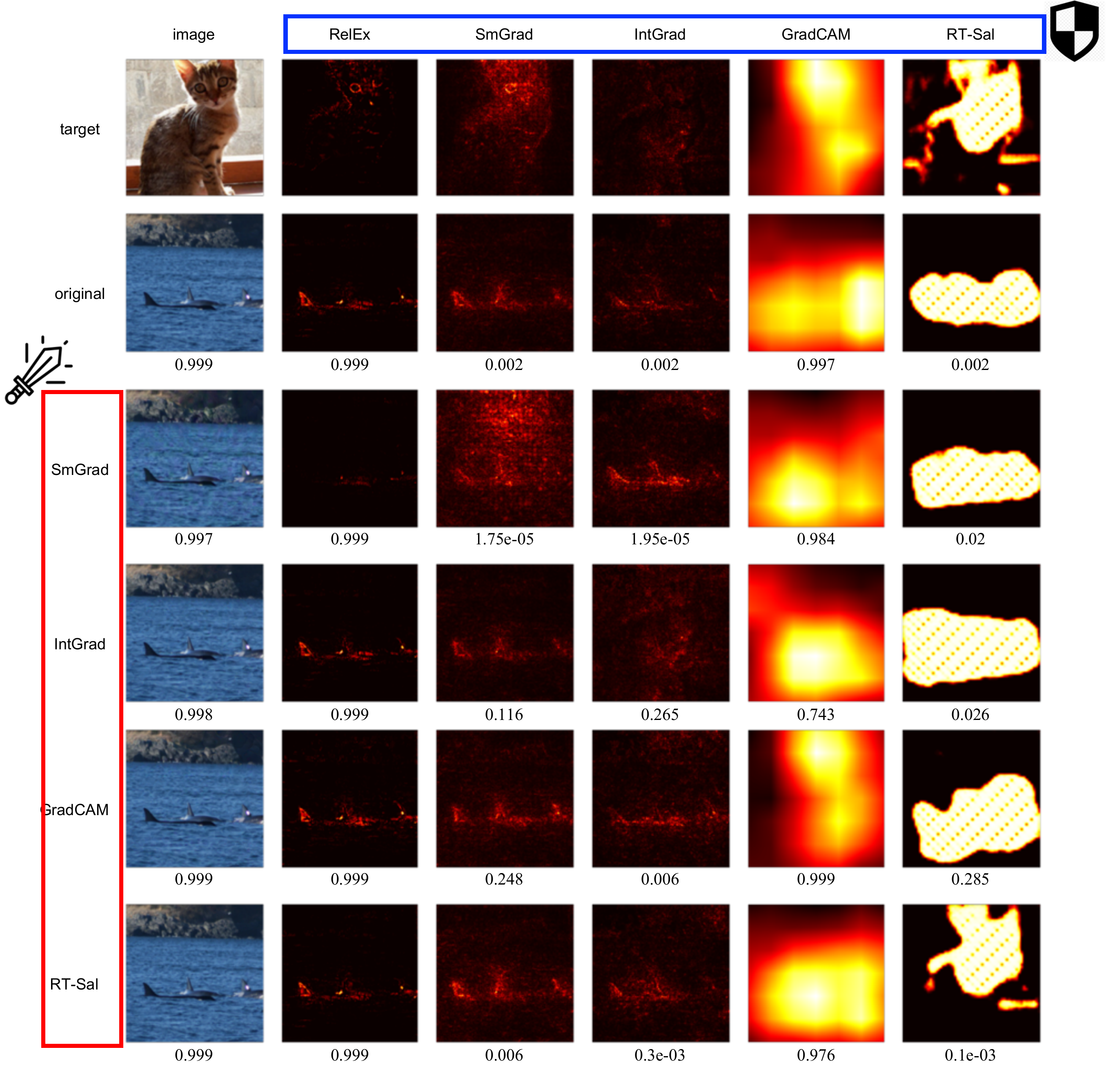}
        \caption{Qualitative results of the methods on the \textbf{natural ResNet-50} against the \textbf{structured attack} (continued). }
    \end{center}
\end{figure}

\sectionbreak
%
%
\subsubsection{The Results of the Targeted Attack against the Natural ResNet-50: Unstructured Attack}

The unstructure attack aims to change the saliency of an original image to that of the \emph{target} image.
We applied the unstructured attack to SmGrad, IntGrad, GradCAM, and RT-Sal (the red box), creating adversarial images for each method being attacked.
We chose one method being attacked, of which saliency maps appear to change significantly.
Then, we extracted explanations of the adversarial images against the selected method by using other methods (the blue box) as shown below.
We present the selected method at the bottom of each figure.
Our method RelEx generates consistent saliency maps with high target class scores of corresponding explanations unlike other methods.

\begin{figure}[H]
    \begin{center}
        \includegraphics[width=0.55\linewidth]{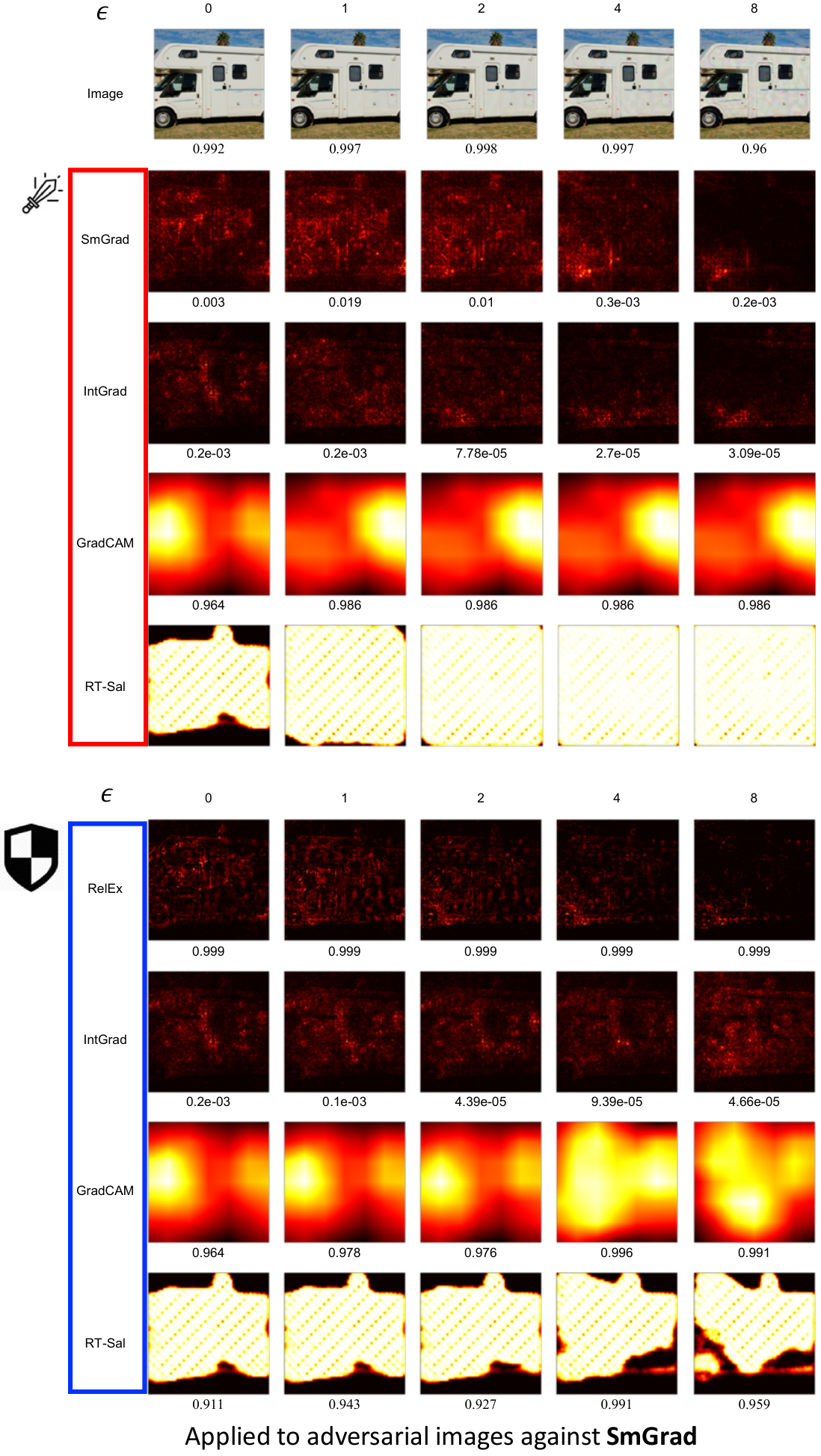}
        \caption{ Qualitative results of the methods on the \textbf{natural ResNet-50} against the \textbf{unstructured attack}. }
    \end{center}
\end{figure}

\begin{figure}[H]\ContinuedFloat
    \begin{center}
        \includegraphics[width=0.7\linewidth]{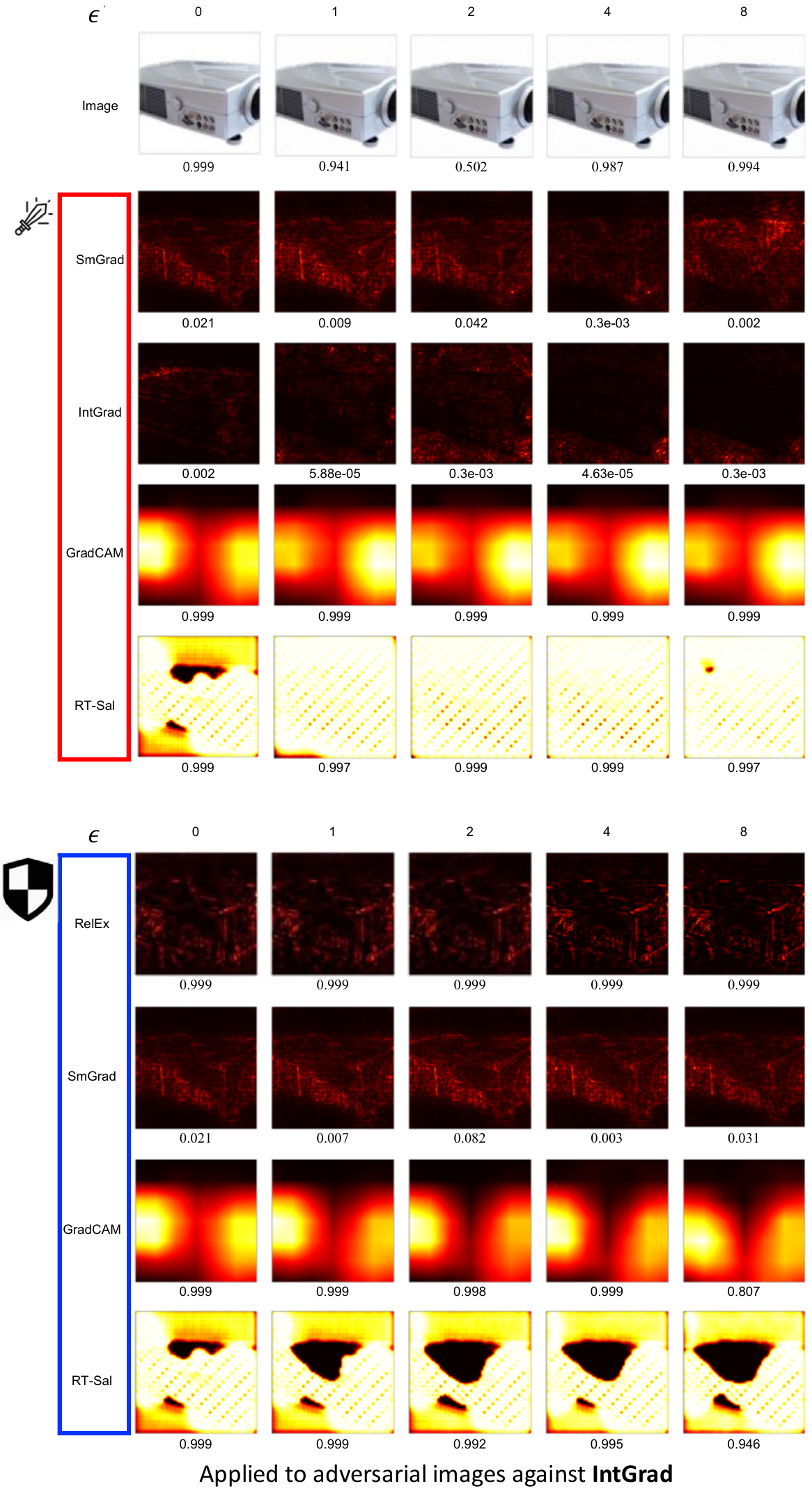}
        \caption{ Qualitative results of the methods on the \textbf{natural ResNet-50} against the \textbf{unstructured attack} (continued). }
    \end{center}
\end{figure}

\begin{figure}[H]\ContinuedFloat
    \begin{center}
        \includegraphics[width=0.7\linewidth]{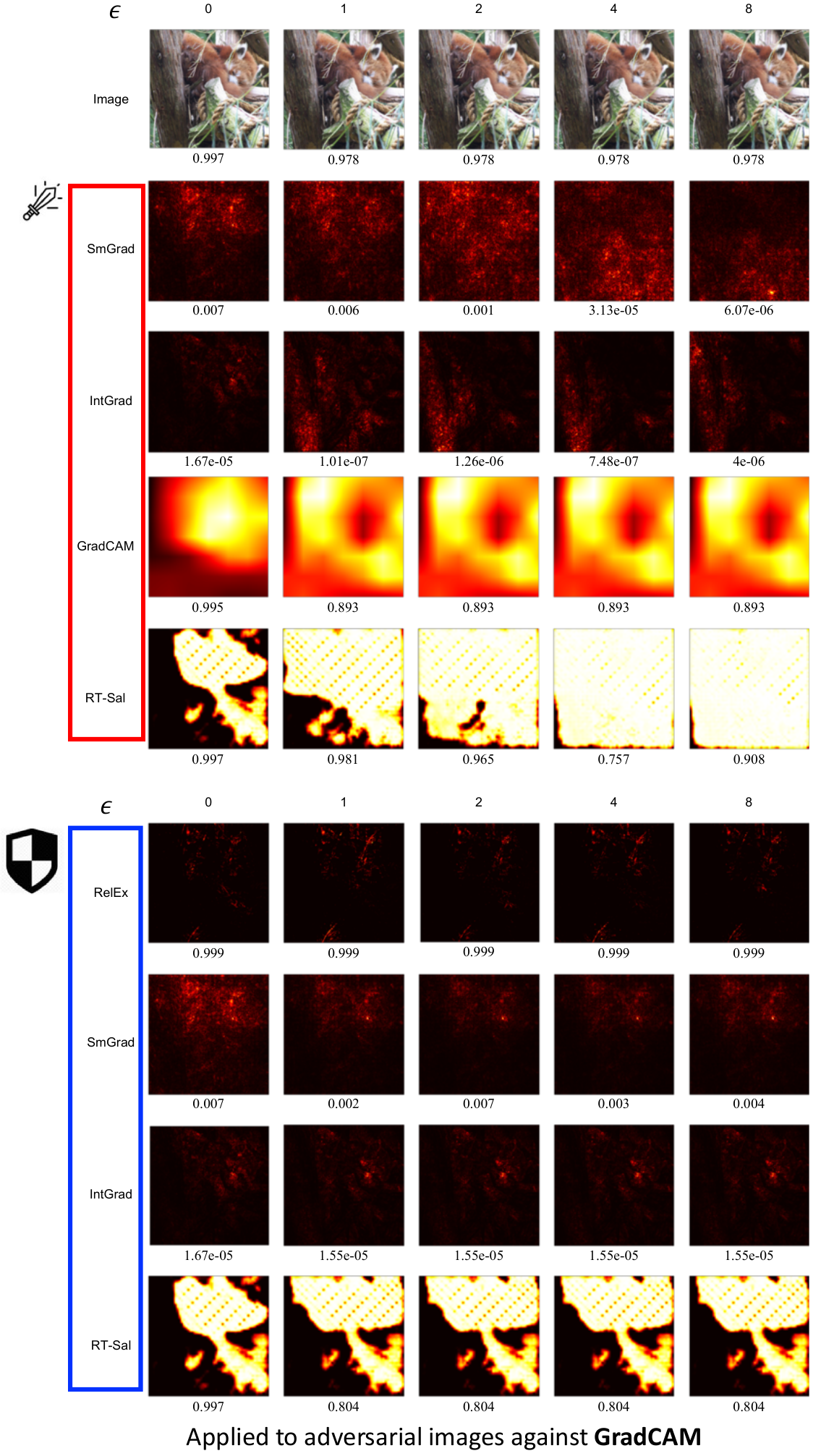}
        \caption{ Qualitative results of the methods on the \textbf{natural ResNet-50} against the \textbf{unstructured attack} (continued). }
    \end{center}
\end{figure}

{\small
\bibliographystyle{ieee_fullname}
\bibliography{manuscript}
}